\pdfoutput=1
\documentclass[letterpaper, 10 pt, journal, twoside]{IEEEtran}

\usepackage{times}
\usepackage[pdftex]{graphicx}
\usepackage{subfigure}
\usepackage{amsmath,amssymb,amsopn,amstext,amsfonts}
\usepackage{cancel}
\usepackage[space]{cite}
\usepackage{soul}
\usepackage{cuted}
\usepackage{lipsum,multicol}
\usepackage{stfloats}
\everymath{\displaystyle}
\usepackage{booktabs}
\usepackage{threeparttable}
\usepackage{amsthm}
\usepackage{xfrac}
\usepackage{balance}
\usepackage{color}
\usepackage{mathtools}

\usepackage{tabularx}
\usepackage{bm}
\usepackage{diagbox}
\usepackage{float}
\usepackage{epstopdf}
\usepackage{url}
\usepackage{multirow}
\usepackage{multicol}
\usepackage{tikz}
\usepackage[linkcolor=black,citecolor=black,urlcolor=black,colorlinks=true]{hyperref}
\usepackage{enumitem}
\usepackage[ruled,vlined,linesnumbered]{algorithm2e}
\usepackage{pifont}
\usepackage{siunitx}
\usepackage{pdfpages}
\DeclareMathAlphabet\mathbfcal{OMS}{cmsy}{b}{n}

\let\oldFootnote\footnote
\newcommand\nextToken\relax

\renewcommand\footnote[1]{%
    \oldFootnote{#1}\futurelet\nextToken\isFootnote}
\newcommand\isFootnote{%
    \ifx\footnote\nextToken\textsuperscript{,}\fi}

\newtheorem{theorem}{Theorem}

\newtheorem{lemma}{Lemma}
\newtheorem{assumption}{Assumption}
\newtheorem{remark}{Remark}

\definecolor{celadon}{rgb}{0.78, 0.93, 0.80}
\pagecolor{celadon}
\nopagecolor
\soulregister\cite7
\soulregister\citep7
\soulregister\citet7
\soulregister\ref7
\soulregister\it7
\soulregister\pageref7

\usepackage{arydshln}
\DeclareGraphicsExtensions{.pdf,.png,.jpg,.eps}

\DeclareMathOperator{\atan}{atan}

\IEEEoverridecommandlockouts

\title{Occupancy Grid Mapping without Ray-Casting for High-resolution LiDAR Sensors}

\author{Yixi Cai, Fanze Kong, Yunfan Ren, Fangcheng Zhu, Jiarong Lin, Fu Zhang

    \thanks{Manuscript received February 5, 2023; revised July 17, 2023; accepted September 29, 2023. This work is supported by the University Grants Committee of Hong Kong General Research Fund under Project 17206421 and DJI Donation. \textit{(Yixi Cai and Fanze Kong contributed equally to this work.)} \textit{(Corresponding author: Fu Zhang.)}}
    \thanks{The authors are with Mechatronics and Robotic Systems (MaRS) Laboratory, Department of Mechanical Engineering, University of Hong Kong, Hong Kong SAR, China. (email: yixicai@connect.hku.hk; kongfz@connect.hku.hk; renyf@connect.hku.hk; zhufc@connect.hku.hk; jiarong.lin@connect.hku.hk; fuzhang@hku.hk)}
}%

\begin{document}
\maketitle
\begin{abstract}

	Occupancy mapping is a fundamental component of robotic systems to reason about the unknown and known regions of the environment. This article presents an efficient occupancy mapping framework for high-resolution LiDAR sensors, termed D-Map. The framework introduces three main novelties to address the computational efficiency challenges of occupancy mapping. Firstly, we use a depth image to determine the occupancy state of regions instead of the traditional ray-casting method. Secondly, we introduce an efficient on-tree update strategy on a tree-based map structure. These two techniques avoid redundant visits to small cells, significantly reducing the number of cells to be updated. Thirdly, we remove known cells from the map at each update by leveraging the low false alarm rate of LiDAR sensors. This approach not only enhances our framework's update efficiency by reducing map size but also endows it with an interesting decremental property, which we have named D-Map. To support our design, we provide theoretical analyses of the accuracy of the depth image projection and time complexity of occupancy updates. Furthermore, we conduct extensive benchmark experiments on various LiDAR sensors in both public and private datasets. Our framework demonstrates superior efficiency in comparison with other state-of-the-art methods while maintaining comparable mapping accuracy and high memory efficiency. We demonstrate two real-world applications of D-Map for real-time occupancy mapping on a handle device and an aerial platform carrying a high-resolution LiDAR. In addition, we open-source the implementation of D-Map on GitHub to benefit society: \href{https://github.com/hku-mars/D-Map}{\tt github.com/hku-mars/D-Map}.
	
\end{abstract}


\section{Introduction}\label{sec:intro}
\IEEEPARstart{I}{n} recent years, advancements in light detection and ranging (LiDAR) sensors have led to the commercialization of lightweight, low cost, and high accuracy 3D LiDARs, raising tremendous popularity in various applications such as robotics\cite{gao2019flying,ren2022bubble,ren2022online,zhu2022swarmlio}, autonomous driving\cite{leonard2008perception,levinson2011towards,li2020deep}, 3D reconstruction\cite{zhang2014loam,lin2023immesh, liu2023HBA}, etc. Over the past decade, a clear trend has emerged to develop 3D LiDARs with a smaller size, longer detection range, and higher resolution, approaching image-level quality\cite{li2020lidartrend}. This trend has not only opened up new possibilities for deploying LiDARs in diverse applications but has also necessitated the development of new techniques to exploit the potential of LiDAR-based applications.

\begin{figure}[t]
	\setlength\abovecaptionskip{-0.1\baselineskip}
	\centering
	\includegraphics[width=\linewidth]{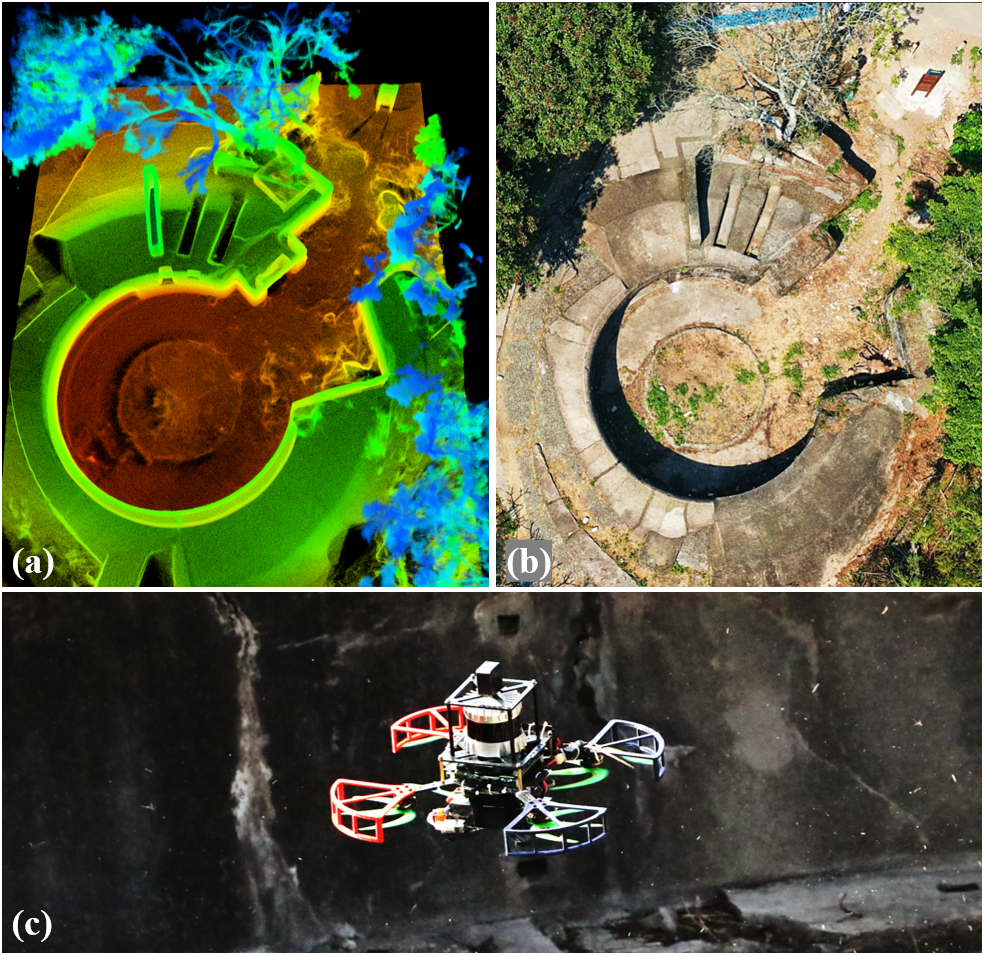}
	\caption{Our proposed framework D-Map served as a real-time high-resolution occupancy mapping module for an autonomous UAV exploration task in an ancient fortress. (a) The high-fidelity point cloud collected by UAV. (b) A bird-view of the scene. (c) The aerial platform carried a 128-channel LiDAR (OS1-128) to conduct the exploration task. The accompanying video of this paper is available on Youtube:  \href{https://youtu.be/m5QQPbkYYnA}{\tt youtu.be/m5QQPbkYYnA}.}
	\label{fig:coverfigure} 
\end{figure} 

Occupancy grid mapping provides an efficient means for autonomous systems to navigate in unknown environments, which is crucial in a variety of applications such as obstacle avoidance\cite{lopez2017aggressive,kong2021avoiding}, path planning\cite{liu2016high,tordesillas2019faster}, and autonomous exploration\cite{zhang2020fsmi,zhou2021fuel,tabib2021autonomous}. When adapted to LiDARs with increasing performance, occupancy grid mapping is faced with additional challenges in efficiency, especially in the two core steps: ray-casting and occupancy state updates.

Ray-casting is challenging because LiDARs provide increasingly dense depth measurements (e.g., over one million points per second) with long ranges (e.g., over \SI{300}{\meter}). Therefore, the number of rays to be processed increases with the number of points. Secondly, the number of cells traversed by a single ray is affected by the sensor's detection range, with LiDARs traversing dozens of times more cells than depth cameras with smaller detection range (e.g., less than \SI{5}{\meter}). Consequently, the tremendous amount of cells results in a costly computation load. Finally, the dense LiDAR measurements can lead to redundant visits to the same cells in ray-casting, which can consume unnecessary computation resources in the subsequent occupancy state updates.

The second challenge in LiDAR-based occupancy mapping arises from the need for high-resolution occupancy maps to fully exploit the high accuracy of LiDAR sensors. This challenge involves the selection of map structures to achieve high-resolution occupancy mapping where a trade-off between computational and memory efficiency has to be balanced. The two commonly used map structures are grid-based\cite{Alberto1996} and tree-based\cite{hornung2013octomap}. Grid-based maps offer highly efficient updates in $\boldsymbol{\mathcal{O}}(1)$ time complexity but suffer from extensive memory consumption when applied in large-scale environments or with high resolutions. Tree-based maps are more memory-efficient than grid-based maps, but the trade-off is the higher computational complexity when updating the tree. The update efficiency of a tree-based map structure is directly influenced by the tree height, which is determined by both the mapping environment and map resolution.
 
In addition to the challenges mentioned above, some essential features of LiDAR sensors have not yet been fully utilized in existing occupancy mapping approaches, such as the low false alarm rate which provides high confidence in identifying free and occupied space. The existing occupancy mapping approaches commonly utilize occupancy probabilities to handle the false detection of depth measurements. Nevertheless, in LiDAR-based occupancy mapping, the low false alarm rate  {(e.g., three out of a million points \cite{manual2020avia})} offers an opportunity to enhance efficiency by directly eliminating known space from the map without requiring any probability updates.

In this work, we present a novel mapping framework that effectively addresses the aforementioned issues in LiDAR-based occupancy mapping. Our contributions can be summarized as follows:
\begin{itemize}
    \item We propose an occupancy state determination method based on depth image projection to alleviate the computation load in the traditional ray-casting technique. This projection-based approach enables occupancy state determination of cells at any size, allowing subsequent efficient updates in large-scale environments.
    \item We present a novel on-tree update strategy for updating occupancy states based on a hybrid map structure, providing a superior balance between computation and memory efficiency. The hybrid map structure stores unknown space on an octree, which enables a memory-efficient representation for the large unknown space, while the occupied space is stored on a hashing grid map. In terms of efficiency, the proposed strategy allows occupancy state determination of large cells on an octree, thereby avoiding unnecessary updates on small cells and increasing efficiency. 
    \item We leverage the low false alarm rate of LiDAR measurements to directly remove cells with determined states (i.e., occupied or free) at each update. This approach renders our map structure a decremental property, for which we term our framework as D-Map, providing higher computational efficiency and less memory usage.
    \item We conduct an in-depth analysis of the accuracy of the proposed occupancy state determination method and the time complexity of updates and queries in D-Map. Specifically, we derive an analytical function to quantify the accuracy loss relative to depth image resolution. The time complexity analysis of updates on D-Map provides theoretical support to our superior performance over state-of-the-art methods that rely on ray-casting.
    \item We make the implementation of D-Map available on Github: \href{https://github.com/hku-mars/D-Map}{\tt github.com/hku-mars/D-Map} to promote the reproducibility and further development of our work. 
\end{itemize}

The remainder of the paper is organized as follows. Section \ref{sec:related_work} provides an overview of related works in the field of occupancy mapping. We then present the complete mapping framework and the details of each key component in Sections \ref{sec:overview}, \ref{sec:method}, and \ref{sec:mapping}, respectively. Section \ref{sec:benchmark} presents the benchmark experiments and ablation studies conducted on open datasets. In Section \ref{sec:experiment}, we demonstrate two real-world experiments, followed by a discussion in Section \ref{sec:discussion}. Finally, we conclude this article in Section \ref{sec:conclusion}.

\section{Related Works}\label{sec:related_work}
In this section, we review the previous research on occupancy mapping and discuss their update methods. 

\subsection{Occupancy Mapping Approaches}
We categorize the occupancy mapping approaches into two classes: continuous maps and discrete maps, based on their assumptions of modeling the environments. 

Continuous maps assume an implicit correlation of locality in space and model environments through continuous occupancy distribution. In the early years, Thrun \textit{et al.} introduced the concept of learning inverse sensor models to generate local occupancy maps by means of sampling\cite{thrun2002probabilistic}. In more recent years, O'Callaghan \textit{et al.} proposed using Gaussian Process regression to learn a continuous representation of 2-dimensional environments, which has since been extended to 3-dimensional environments using Gaussian mixture models or sparse Gaussian processes to reduce computational load during training \cite{o2012gaussian, kim2012GMM, kim2015gpmap}. However, the time complexity of these methods is often prohibitive for real-time applications, with a computational complexity of $\boldsymbol{\mathcal{O}}(N^3)$ in the number of depth measurements $N$. Therefore, considerable efforts have been devoted to improving the update efficiency of these methods \cite{hilbertmap, incrementalhilbertmap, o2018variable, doherty2019BGK}. Unlike previous approaches maintaining dense occupancy states, Duong \textit{et al.} developed a sparse Bayesian formulation for occupancy mapping using a sparse set of relevance vectors\cite{Duong2022Bayesian}. While the aforementioned methods can be applied to LiDAR sensors, there is a substantial body of research focused on learning-based occupancy mapping using image sequences, including Occupancy Networks\cite{occupancynetwork}, Neural Radiance Fields (NeRF)\cite{mildenhall2021nerf}, DeepSDF\cite{park2019deepsdf}, Semantic Mapping\cite{SemanticOctree, 2020Kimera}, NICE-SLAM\cite{zhu2022nice}, among others. 

As opposed to continuous maps, occupancy grid mapping approaches assume independent occupancy at different locations, resulting in a discrete representation of environments using grids. Historically, a uniform grid map structure was imposed to discretize the environments and represents the occupancy probabilities independently\cite{Alberto1996,martin1996robot,roth1989building}. However, this approach is impractical for generating high-resolution maps or working in large-scale environments for its significant memory consumption. To address this limitation, tree-based map structures such as quadtree\cite{hunter1979quadtree} and octree\cite{jackins1980octree} were applied in occupancy grid mapping approaches with great effect. These structures recursively divide space into equal subspaces for updates and merge subspaces with equivalent states for a more compact representation\cite{payeur1997probabilistic,hornung2013octomap}. However, the updates on the tree-based maps are much more expensive. A voxel hashing technique \cite{niessner2013hashmap} has recently been adapted to occupancy mapping\cite{zhou2020ego,zhou2023racer}. Though alleviating the memory consumption in uniform grids, the memory consumption remains prohibitive to high-resolution maps and large-scale environments. The use of discrete representations of environments has also been applied in Euclidean Signed Distance Field (ESDF), an alternative approach to occupancy mapping \cite{voxblox, fiesta}. However, the construction of ESDF heavily relies on ray-casting and encounters similar challenges to other existing occupancy mapping techniques.

 While the continuous mapping approaches provide dense occupancy maps that can reason about measured and unmeasured areas, concerns remain about the reliability of the inferred regions, particularly in safety-critical robotics applications. Besides, continuous mapping approaches face significant challenges in real-time ability due to the complicated training and querying process, as integration over non-trivial space is required to update and query on a continuous map which necessitates timely numerical evaluation. Compared with the continuous mapping approaches methods, occupancy grid mapping is preferable in robotic applications with limited computation resources because of its higher update efficiency resulting from less complex map representation. Additionally, occupancy grid mapping is regarded as a more reliable technique, as it updates the occupancy map solely based on actual depth measurements without any inferences. Our D-Map is one kind of occupancy grid mapping approach. It achieves higher computational efficiency than existing grid-based approaches while maintaining comparable memory consumption with tree-based methods.

\subsection{Update Methods}
Regardless of different map representations, various research has been conducted to improve the efficiency of map updating, most of which focuses on alleviating the computation load of ray-casting. Wurm \textit{et al.} present a hierarchical structure of octrees with multiple resolutions, allowing not complete but adequate resolution in different sub-maps depending on objects to reduce the computation load in ray-casting\cite{wurm2011hierarchies}. Similarly, \cite{einhorn2011finding} propose an adaptive resolution mapping method based on statistical measurements. However, the resultant maps generated from this method are affected by the parameters of adaptive resolution, which may differ from the original results. In Octomap\cite{hornung2013octomap}, a batch-based method is utilized to avoid redundant cells resulting from ray-casting, which effectively reduces the effort in subsequent tree updates. Besides, a clamping policy is used in Octomap, initially proposed to handle dynamic environments\cite{yguel2008update}, allowing the acceleration of map updates by pruning the octree. Duberg \textit{et al.} proposed two possible approaches to simplify the ray-casting process, including downsampling point clouds to map resolution and ray marching with adaptive steps. However, both of these approaches introduce accuracy loss and require handcraft parameters\cite{duberg2020ufomap}. In contrast, motivated by the idea of coherent rays\cite{wald2006ray}, Kwon \textit{et al.} combine rays traversing equivalent cells into a super ray to reduce the number of rays in ray-casting\cite{kwon2019superray} while maintaining full accuracy. Besides super rays, the number of cells to be traversed is reduced by culling regions. However, these two techniques provide less benefit at high resolution due to the ineffectiveness of merging rays into a super ray and the inefficiency of culling region construction.

In comparison with the aforementioned update approaches, D-Map discards the timely ray-casting for occupancy mapping. Alternatively, we design an occupancy state determination method based on depth image projection. We develop an on-tree update strategy for efficient map updates. This strategy determines the occupancy states of cells at various resolutions by projecting cells to the depth image, which is more reliable compared to the adaptive resolution mapping approach\cite{einhorn2011finding} that relies on handcraft parameters. Moreover, unlike existing occupancy mapping methods, D-Map does not require occupancy probability updates. Instead, we assume the environment is static and directly remove known space (i.e., free and occupied) from the tree map by leveraging the low false alarm rate of LiDAR sensors. This approach results in a decreasing size of the map during the update process, thus further enhancing the efficiency of D-Map.

\section{Overview}\label{sec:overview}
\begin{figure*}[t]
	\setlength\abovecaptionskip{-0.1\baselineskip}
	\centering
	\includegraphics[width=0.95\textwidth]{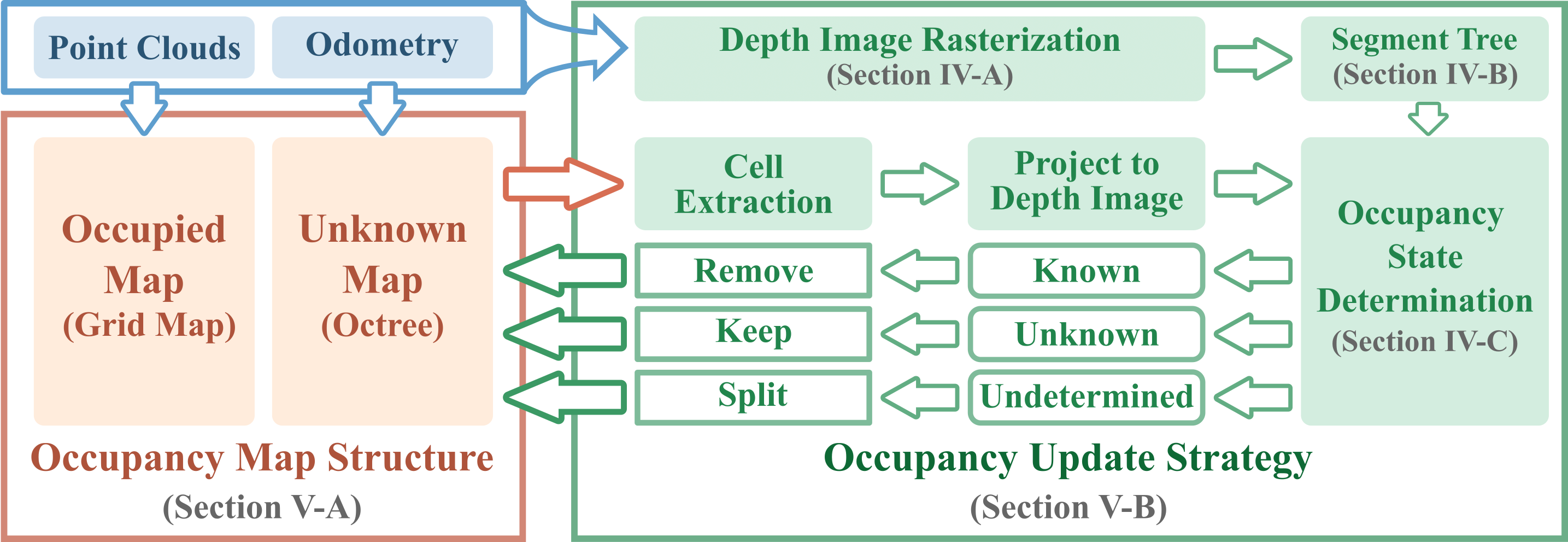}
	\caption{{The framework overview of D-Map. The blue block shows the input to D-Map, including the point clouds and the corresponding sensor odometry. The orange block is the occupancy map structure of D-Map, which is composed of a hashing grid map for maintaining occupied space and an octree for maintaining unknown space. The occupancy update strategy is presented in the green block, which extracts the cells inside the sensing area on the octree and conducts operations depending on the occupancy state determination method using a depth image. }}
	\label{fig:overview} 
 \vspace{-0.4cm}
\end{figure*} 

Figure \ref{fig:overview} presents an overview of the proposed D-Map framework. The map structure delineated in the orange block is explained in Section~\ref{subsec:map}. It consists of two parts: the occupied map and the unknown map. The green block represents the pipeline of the occupancy update strategy. At each update, a depth image is rasterized from the incoming point clouds at the sensor pose (see Section~\ref{subsec:proj}). Subsequently, a 2-D segment tree is constructed on the depth image to enable efficient occupancy state determination (see Section~\ref{subsec:segtree} and Section~\ref{subsec:determination}). The entire procedure to update the unknown map is described in Section~\ref{subsec:update} and summarized as follows. The cell extraction module retracts the unknown cells on the octree from the largest to the smallest size, projects them to the depth image, and determines their occupancy states. The cells determined as known are directly removed from the map, while the unknown ones remain, and the undetermined ones are split into smaller cells for further occupancy state determination. This coarse-to-find process facilitates updates on large cells directly without queries each small cell.  Moreover, the removal of known cells endows our framework with a decremental property, providing high efficiency in both computation and memory. 

\section{Occupancy State Determination on Depth Image}\label{sec:method}
This section describes how to determine the occupancy states on a depth image rasterized from incoming point clouds. 
\subsection{Depth Image Rasterization}\label{subsec:proj}
In preparation for occupancy state determination, the point cloud captured by a LiDAR sensor is rasterized into a depth image at the current sensor pose. To ensure the accuracy of state determination, the depth image resolution should be sufficiently small so that the projected area of a cell from the map to the depth image is larger than one pixel. As illustrated in Fig.~\ref{fig:depthres}, we determine the depth image resolution $\psi_{\mathtt{map}}$ relative to the map resolution $d$ and LiDAR's detection range $R$ using the following equation:
\begin{equation}\label{eq:res_map}
\psi_{\mathtt{map}} = 2\arcsin(\frac{d}{2R}) \approx \frac{d}{R}
\end{equation}
However, high-resolution maps would result in a high-resolution depth image of enormous size, with many empty pixels due to the much smaller number of point clouds than the size of the depth image. To address this issue, we bound the depth image resolution by the LiDAR angular resolution, which is the minimum angle between two laser pulses emitted and received in a rotating manner. Specifically, we define the standard depth image resolution $\psi_I$ as
\begin{equation}\label{eq:res}
\psi_I = \max\{\psi_{\mathtt{map}}, \psi_{\mathtt{lidar}}\}\approx \max\{\frac{d}{R}, \psi_{\mathtt{lidar}}\} \
\end{equation}
where $\psi_{\mathtt{lidar}}$ is the angular resolution of the LiDAR. Note that we do not distinguish between the vertical and horizontal angular resolution for simplicity of definition. Moreover, we keep the minimum depth value when neighboring points are projected to the same pixel.
\begin{figure}[t]
	\setlength\abovecaptionskip{-0.1\baselineskip}
	\centering
	\includegraphics[width=0.9\linewidth]{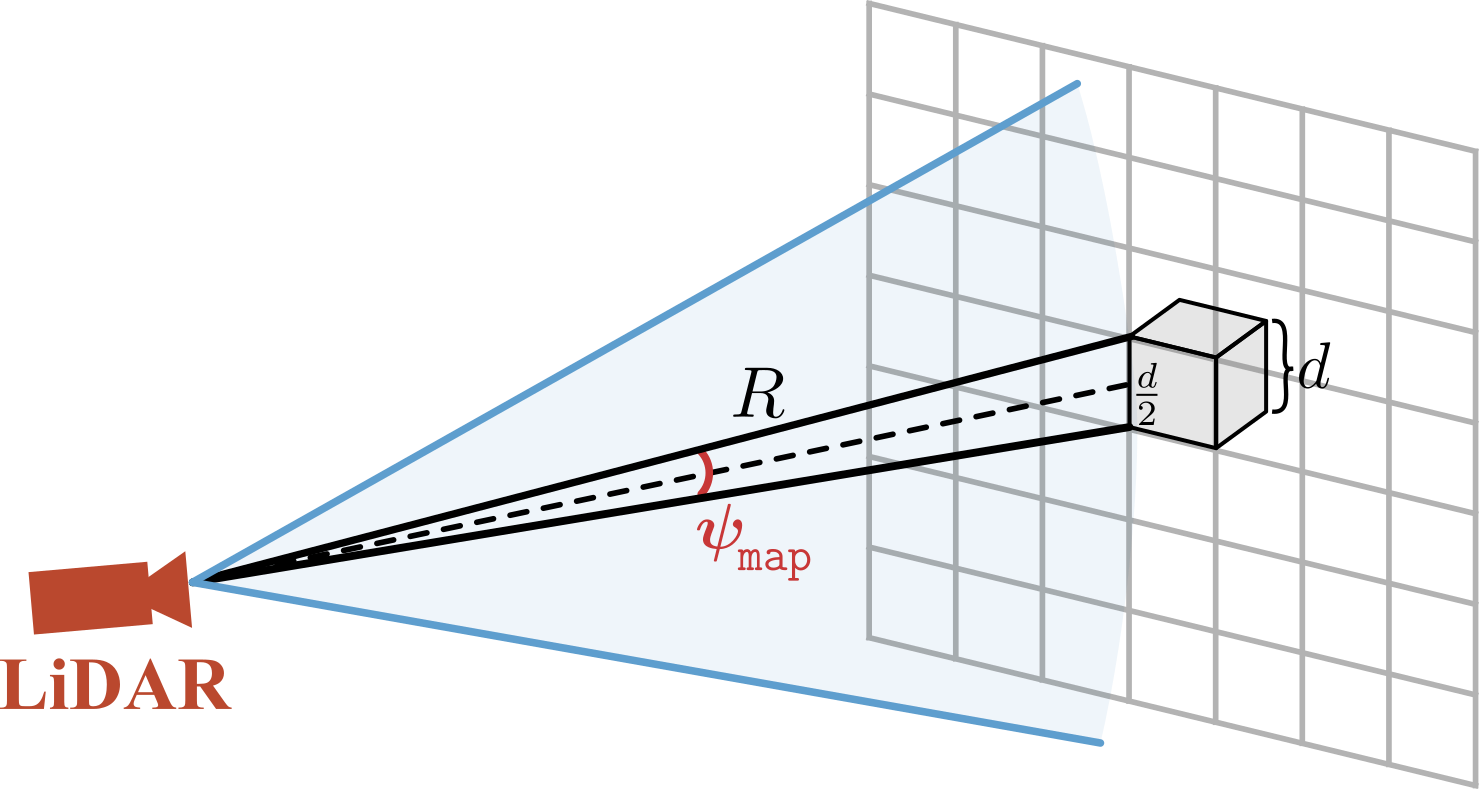}
	\caption{This figure illustrates the spatial relationship among the map resolution $d$, the detection range $R$, and the depth image resolution $\mathrm{\psi_{\mathtt{map}}}$.}
	\label{fig:depthres} 
 \vspace{-0.3cm}
\end{figure} 
 
\subsection{2-D Segment Tree}\label{subsec:segtree}
	\begin{figure}[t]
		\setlength\abovecaptionskip{-0.1\baselineskip}
	\centering
	\includegraphics[width=0.95\linewidth]{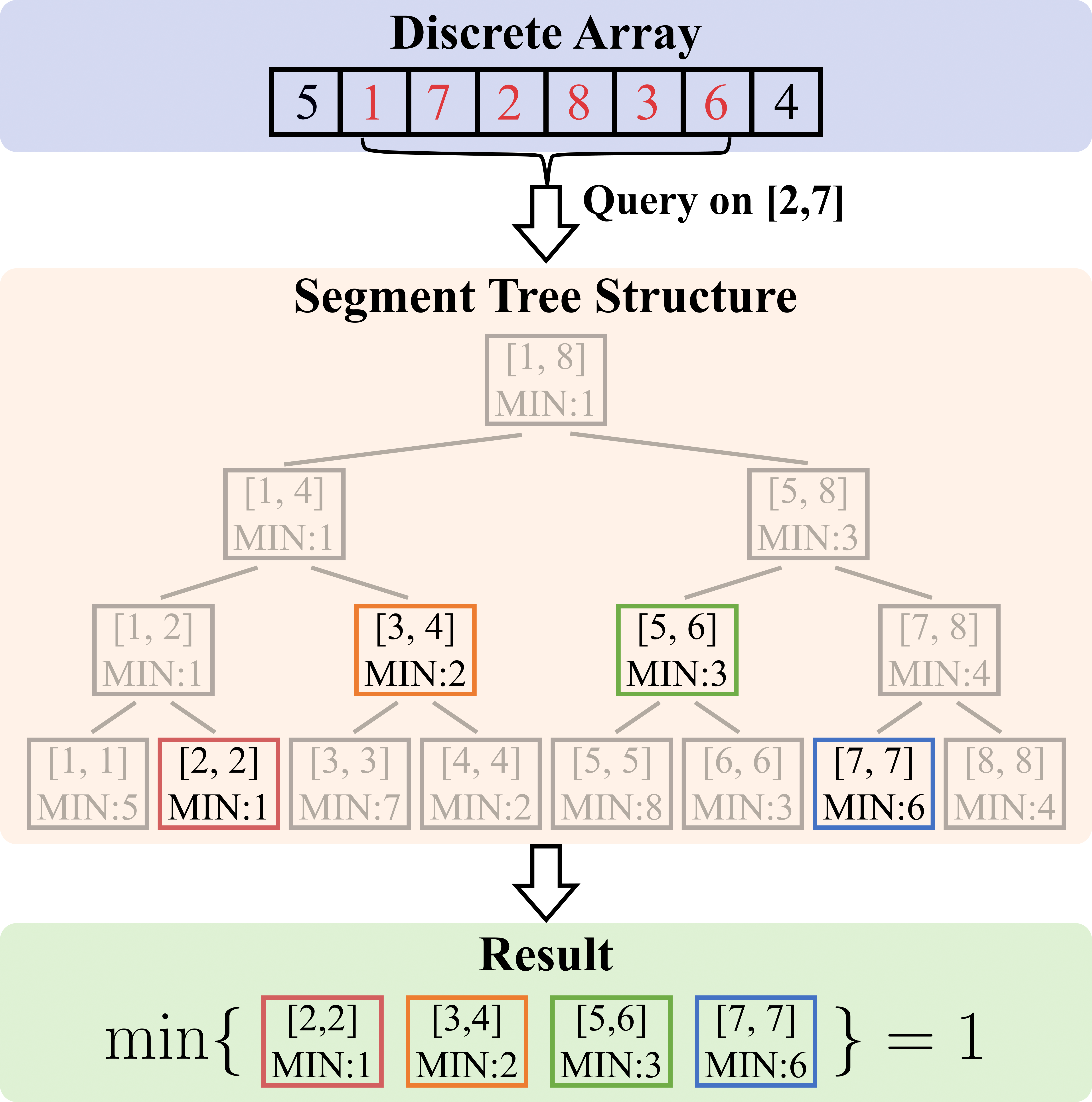}
	\caption{This figure illustrates an example of a fast query of the minimum value in the pixel range of $[2,7]$ on a 1-D segment tree. Starting from the root of the segment tree, the range query searches along the tree recursively until the current node range is completely covered by the queried range, where the minimum value of the node range that has been saved on the node during the tree construction will be returned. In this example, the range $[2,7]$ leads to four nodes representing the range of $[2,2]$, $[3,4]$, $[5,6]$, and $[7,7]$, respectively. The minimum value of the range is efficiently obtained from these four nodes instead of counting the six elements in the array.}
	\label{fig:seg}
 \vspace{-0.4cm}
	\end{figure} 

To determine the occupancy state of a cell in the map, a two-step process is employed, whereby the cell is first projected onto the depth image, followed by a comparison of the projected depth against the minimum and maximum depth values of the corresponding area. Since the projected area of a cell on the depth image varies with the cell location, a 2-D segment tree structure is employed to expedite efficient queries of the minimum and maximum values on the depth image, as detailed below.

A segment tree is a perfectly balanced binary tree that efficiently provides range queries by representing a set of intervals \cite{bentley1980segtree}. Figure~\ref{fig:seg} describes the process of querying the minimum value by a 1-D segment tree. The segment tree is constructed by recursively splitting the array in half until each node contains a single element. As each node on the segment tree represents an interval of the array, the summarized information of the values in the interval, such as the minimum and maximum, is preprocessed during the construction procedure to accelerate the subsequent query. When querying, the segment tree retrieves a minimal representation of the queried interval using a subset of nodes  (colored nodes in Fig. \ref{fig:seg}). The result is obtained by summarizing information from the retrieved nodes with fewer operations than direct queries. The time complexity of query on a 1-D segment tree is $\boldsymbol{\mathcal{O}}(\log N)$ where $N$ is the number of elements in the discrete array, while direct query leads to a time complexity of $\boldsymbol{\mathcal{O}}(N)$.

The approach to extending a 1-Dimensional segment tree to a 2-Dimensional structure involves constructing a ``segment tree of segment trees", as proposed in \cite{vaishnavi1982segtree2D}. The segment tree in the outer layer splits the 2-D array by row, and on each node of the outer segment tree, a 1-D inner segment tree is constructed to maintain the column information on the covered rows. The query on a 2-D segment tree first searches the outer tree for the nodes representing rows covered by the queried region and then traverses the inner segment trees on the corresponding nodes to retrieve the covered columns. Finally, the result over the queried region is summarized from the information stored on retrieved nodes. The time complexity of a 2-D segment tree to query on a 2-D array of size $N\times M$ is $\boldsymbol{\mathcal{O}}(\log N \log M)$, while direct queries lead to a time complexity of $\boldsymbol{\mathcal{O}}(N^2)$.

Given a depth image rasterized from point clouds, we build a 2-D segment tree to maintain the minimum and maximum depth values on each tree node, denoted as $\mathtt{dMin}$ and $\mathtt{dMax}$, respectively. Additionally, we keep track of the number of pixels occupied by point clouds within the covered area of each node, denoted as $\mathtt{dSum}$.  

\subsection{Occupancy State Determination}\label{subsec:determination}
\begin{algorithm}[t]
	\SetAlgoLined
	\SetKwInOut{Input}{Input}
	\SetKwInOut{Output}{Output}
	\SetKwInOut{Param}{Params}
	\SetKwFunction{query}{$\mathcal{T}$$.\mathtt{query}$}
	\SetKwFunction{determine}{$\mathtt{DetermineOccupancy}$}
	\SetKwFunction{proj}{$\mathtt{Projection}$}
	\SetKwFunction{cover}{$\mathtt{Coverage}$}
	\SetKwFunction{return}{$\mathtt{return}$}
	\Param{Completeness threshold $\mathtt{\varepsilon}$}
	\Input{Grid Center $\mathtt{C}$, Grid Size $\mathtt{L}$, LiDAR Pose $\mathtt{T}$,\\ 2D Segment Tree $\mathcal{T}$\\}
	\Output{State $\mathtt{S}$}
	
	\SetKwProg{Fn}{Function}{}{}
	\Fn{\determine{$\mathtt{T}$,$\mathtt{C}$,$\mathtt{L}$,$\mathcal{T}$}}{
		$r,\theta,\phi \leftarrow$ \proj{$\mathtt{T}$,$\mathtt{C}$};\\
		$\mathtt{BoxMin}$ $= r-L/2$, $\mathtt{BoxMax}$ $= r+L/2$;\\    
		$\mathtt{xL,xR,yL,yR}\leftarrow$\cover{$r,\theta,\phi,\mathtt{L}$};\\	
		$\mathtt{ProjPixels}$ = $(\mathtt{yR}-\mathtt{yL}+1)(\mathtt{xR}-\mathtt{xL}+1)$;\\			
		$\mathtt{dMin,dMax,dSum} \leftarrow$ \query{$\mathtt{xL,xR,yL,yR}$};\\
            $\alpha = \mathtt{dSum}/\mathtt{ProjPixels}$;\\
		\lIf{$\mathtt{dSum}==0$}{\return{$\mathtt{Unknown}$}}
		\eIf{$\alpha>\mathtt{\varepsilon}$}{
			\eIf{$\mathtt{dMax}$$<$$\mathtt{BoxMin}$}{
				$\mathtt{S}\leftarrow \mathtt{Unknown}$;
			}{
				\eIf{$\mathtt{dMin}$$>$$\mathtt{BoxMax}$}{
					$\mathtt{S}\leftarrow \mathtt{Unknown}$;
				}{
					$\mathtt{S}\leftarrow \mathtt{Undetermined}$;
				}
			}
		}{	\eIf{$\mathtt{dMax}$$<$$\mathtt{BoxMin}$}{
				$\mathtt{S}\leftarrow \mathtt{Unknown}$;
			}{
				$\mathtt{S}\leftarrow \mathtt{Undetermined}$;
			}		
		}
		\return{$\mathtt{S}$};
	}
	\textbf{End Function}
	\caption{Occupancy State Determination}
	\label{alg:occupancy}
\end{algorithm}
\begin{figure*}[t]
	\setlength\abovecaptionskip{-0.1\baselineskip}
	\centering
	\includegraphics[width=\textwidth]{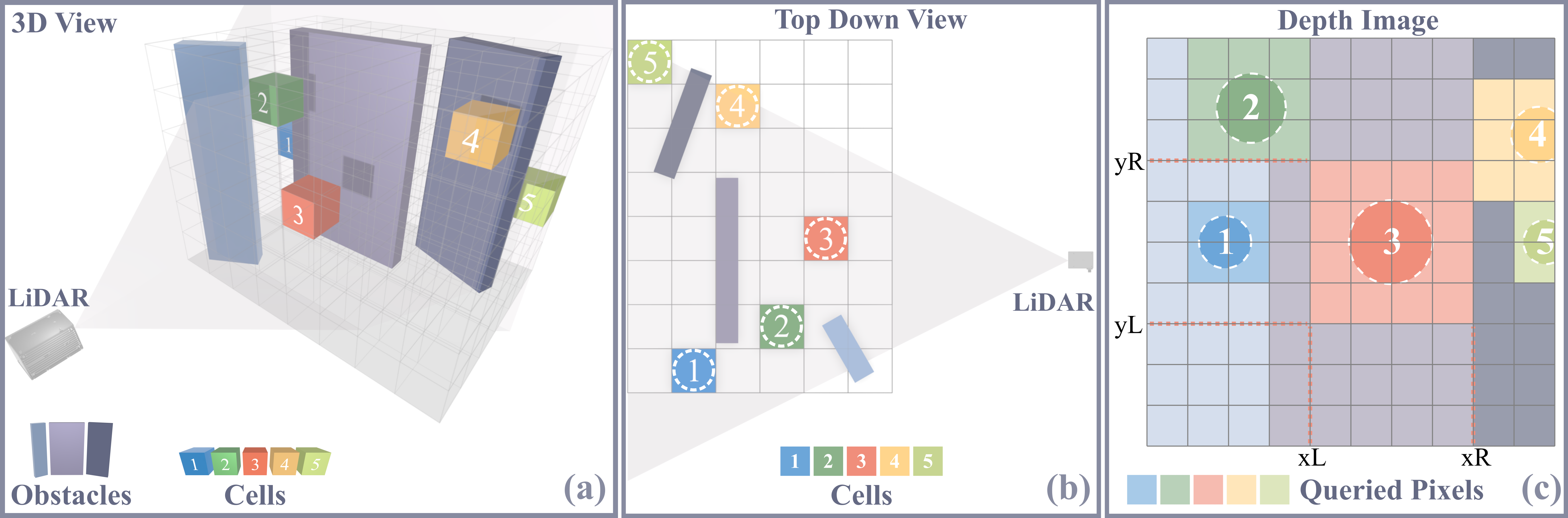}
	\caption{ This figure demonstrates an example of occupancy state determination on five cells in the map. (a) The 3D view shows the relative position of the five cells with respect to (w.r.t.) the LiDAR and the objects in the environment. (b) The top-down view helps to understand the occlusion between the cells and the objects when seeing from LiDAR. (c) The cells are projected to the depth image by their circumsphere radius, after which the projected areas are queried for the depth values in the depth image, which are finally used to determine the cells' occupancy states. }
	\label{fig:states} 
 \vspace{-0.2cm}
\end{figure*} 
\begin{figure}[t]
	\setlength\abovecaptionskip{-0.1\baselineskip}
	\centering
	\includegraphics[width=\linewidth]{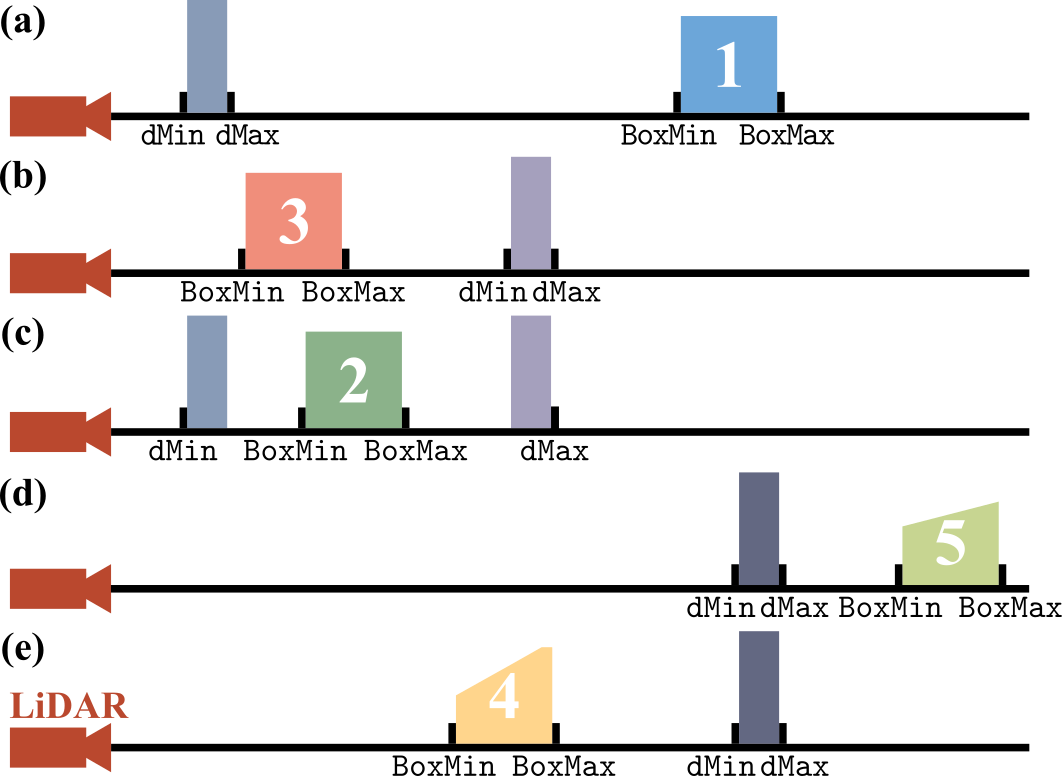}
	\caption{The relative position between the cells and objects in the environments, along the one-pixel direction of the depth image, is determined by comparing the depth range of the cell (represented by the minimum depth $\mathtt{BoxMin}$ and maximum depth $\mathtt{BoxMax}$) with the depth range on the depth image (represented by the minimum depth ($\mathtt{dMin}$ and maximum depth $\mathtt{dMax}$). Grids 1-3 completely locate inside the LiDAR's sensing area, while Grids 4-5 are partially inside. }
	\label{fig:relative} 
 \vspace{-0.5cm}
\end{figure} 

We introduce the principle of our method for determining the occupancy state of a cell using five cells as an example, depicted in Fig.~\ref{fig:states} and numbered from 1 to 5. We commence by classifying the cells based on whether they are completely located inside the LiDAR's sensing area. As illustrated in the top-down view in Fig.~\ref{fig:states}(b), Grid 1, Grid 2, and Grid 3 are entirely located within the sensing area, while Grid 4 and Grid 5 only have part of them inside. Among the cells completely inside, Grid 3 is determined as known since it is situated in front of all objects in the observed environment; Grid 1 is determined as unknown due to its location behind the objects. The occupancy state of Grid 2 remains undetermined since part of it lies in front of the objects while the other part lies behind. Regarding the cells with part of them inside the sensing area, Grid 5 is determined as unknown because it is located behind the objects. Though lying in front of the object, the occupancy state of Grid 4 remains undetermined owing to its location not completely inside.

Following the principles above, D-Map projects the cells that are either inside or intersected with the LiDAR's sensing area onto the current depth image rasterized from recent point clouds which represent the objects in the environment. The relative location between the cells and the objects is determined by comparing their depth values. The procedure of occupancy state determination on the depth image is described in Alg. \ref{alg:occupancy} and explained below.

In consideration of consistency and computational efficiency, we project a cell to the depth image by its inscribed circle, as shown in Fig. \ref{fig:states} (b) and (c). The cell center, denoted as $\mathtt{C}$, is first transformed into a spherical coordinate as $(r,\theta,\phi)$ (Line 2). The depth range covered by the cell is described as its minimum depth $\mathtt{BoxMin}$ and maximum depth $\mathtt{BoxMax}$ with respect to its center $\mathtt{C}$ and cell size $\mathtt{L}$ (Line 3). The boundary of the projected area on the depth image (i.e., queried pixels in Fig. \ref{fig:states}(c)) is obtained as follows 
\begin{equation}
\label{eq:area}
	\begin{aligned}
		\mathtt{x_L} = \lfloor \frac{\theta - \arcsin(\frac{L}{2r})}{\psi_I}\rfloor,~\mathtt{x_R} = \lceil \frac{\theta + \arcsin(\frac{L}{2r})}{\psi_I}\rceil
		\\
		\mathtt{y_L} = \lfloor \frac{\phi - \arcsin(\frac{L}{2r})}{\psi_I}\rfloor,~\mathtt{y_R} = \lceil \frac{\phi + \arcsin(\frac{L}{2r})}{\psi_I}\rceil \\
	\end{aligned}
\end{equation}

Then, we prepare the following information for determining the occupancy states of the cell. Firstly, the number of queried pixels covered by the projected area is counted in $\mathtt{ProjPixels}$ (Line 5). Secondly, the number of occupied pixels $\mathtt{dSum}$, minimum depth $\mathtt{dMin}$, and maximum depth $\mathtt{dMax}$ within the projected area is provided by querying the 2-D segment tree as described in Section~\ref{subsec:segtree} (Line 6). Thirdly, we define the observation completeness $\alpha$ of a cell to determine whether a cell is completely observed by the current depth image, which is calculated by dividing the number of occupied pixels $\mathtt{dSum}$ over the number of queried pixels $\mathtt{ProjPixels}$ on the depth image (Line 7). The computed $\alpha$ quantifies if the cell's projected area has been sufficiently observed by points in the current depth image. A large $\alpha$ indicates that a major part of the cell lies inside the sensing area, with most of the pixels in its projected area on the depth image actively occupied by LiDAR points. Only cells with large $\alpha$ have their occupancy state updated. 

 The occupancy state of a cell is acquired by comparing the depth range of the depth image (represented by the minimum depth $\mathtt{dMin}$ and maximum depth $\mathtt{dMax}$) against the depth range of the cell (represented by minimum depth $\mathtt{BoxMin}$ and maximum depth $\mathtt{BoxMax}$), as shown in Fig.~\ref{fig:relative} and detailed below. We first categorize the cells by comparing the observation completeness $\alpha$ with a completeness threshold $\mathtt{\varepsilon}$. For completely observed cells (i.e., $\alpha>\mathtt{\varepsilon}$ in Line 9), the occupancy state is determined as unknown if $\mathtt{dMax}$ of the depth image is smaller than $\mathtt{BoxMin}$ of the cell (Line 10$\mathrm{\sim}$12). This condition indicates that the cell is located behind the object in the environment, as shown in Fig.~\ref{fig:relative}(a). When $\mathtt{dMin}$ of the depth image is greater than $\mathtt{BoxMax}$ of the cell, the occupancy state of the cell is determined as known (Line 13$\sim$15), indicating that it is located in front of the objects as shown in Fig.~\ref{fig:relative}(b). Otherwise, the occupancy state of the cell cannot be determined, as shown in Fig.~\ref{fig:relative}(c). For cells that are not completely observed (i.e., $\alpha\leqslant\mathtt{\varepsilon}$), the occupancy state is determined as unknown if $\mathtt{dMax}$ of the depth image is less than $\mathtt{BoxMin}$ of the cell (Line 20$\sim$22), indicating that the cell lies behind the objects as shown in Fig.\ref{fig:relative}(d). If this condition is not met, the occupancy state of the cell is considered undetermined due to incomplete observation (Line 22$\sim$24), as shown in Fig.~\ref{fig:relative}(e).

\begin{figure}[t]
	\setlength\abovecaptionskip{-0.1\baselineskip}
	\centering
	\includegraphics[width=0.9\linewidth]{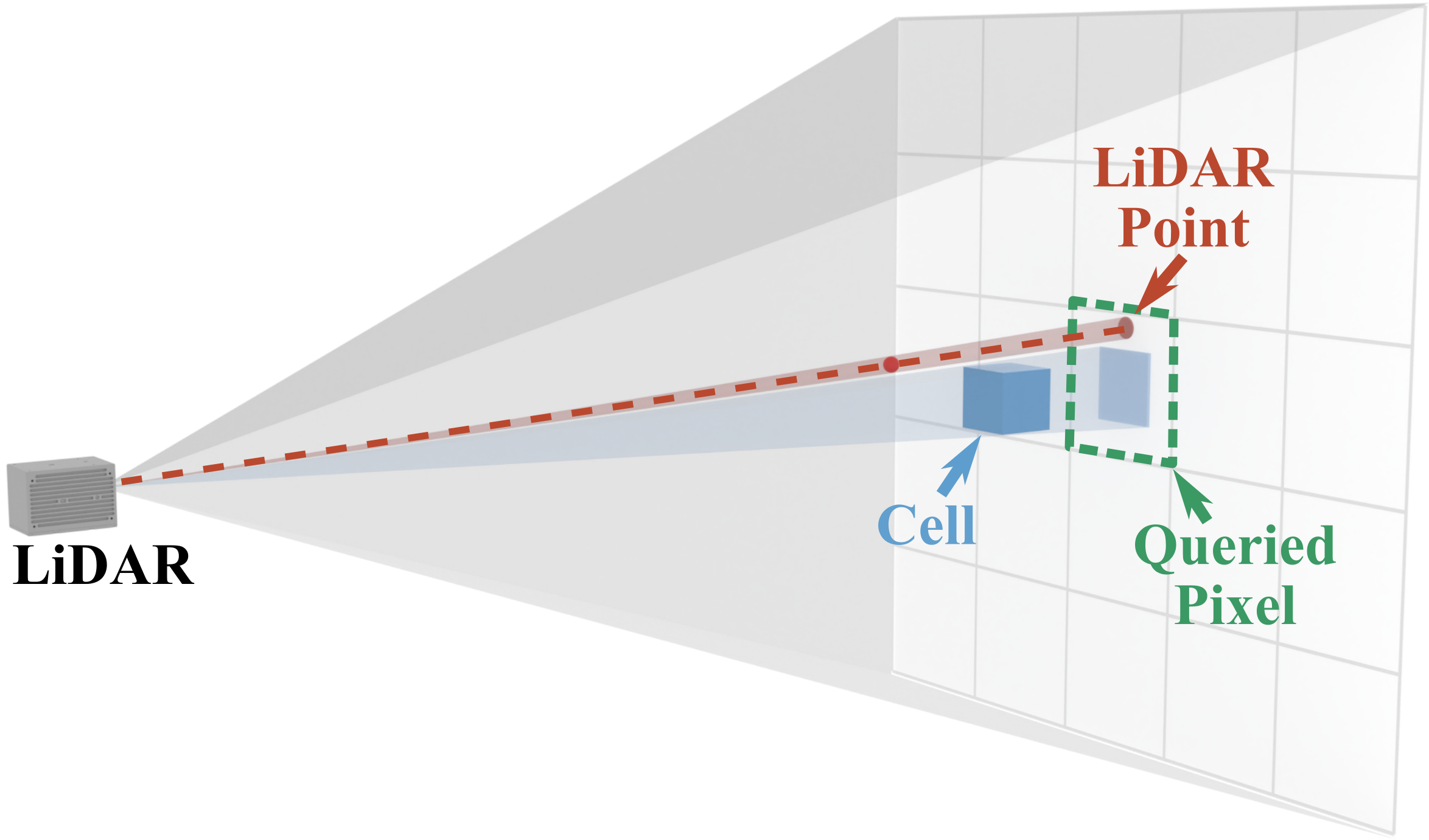}
	\caption{A special case when the pixel size is larger than the projected area of a cell. This happens when the depth image resolution is computed from LiDAR's angular resolution.}
	\label{fig:lidarres} 
 \vspace{-0.5cm}
\end{figure} 
In some situations, the projected area of a cell may be smaller than the one-pixel area on the depth image, as shown in Fig. \ref{fig:lidarres}. This occurs when the depth image resolution $\psi_I$ is computed from the LiDAR's angular resolution, as defined by (\ref{eq:res}). These cases are referred to as ``single pixel cases''. In such cases, recalling that a pixel only saves one LiDAR point with the smallest depth as presented in Section~\ref{subsec:proj}, we should further determine whether the LiDAR point actually traverses the cell by examining if the point lies within the cell's projected area.  If the LiDAR point lies outside the projected area, the occupancy state of the cell remains unknown. Otherwise, we use a similar logic to Alg. \ref{alg:occupancy} to determine the occupancy state. 
\subsection{Depth Image Resolution Analysis}\label{subsec:acc_analysis}
The resolution of a depth image is a critical parameter that plays a crucial role in balancing the accuracy and efficiency of occupancy state determination. A low-resolution depth image is more convenient to query but at the cost of increased information loss from the original point clouds. Conversely, a high-resolution depth image provides better accuracy but requires a larger computational effort, leading to reduced efficiency. When considering a depth image $\mathcal{M}_I$ with a standard resolution of $\psi_I$ determined by (\ref{eq:res}), it is possible to relax the resolution to $\psi = \gamma \psi_I$ by a relaxed factor $\gamma$ to achieve higher efficiency while sacrificing some accuracy. To evaluate the accuracy loss with the resolution relaxation factor $\gamma$, we derive a function $f(\gamma)$ that describes the retained accuracy from the original depth image $\mathcal{M}_I$ as a function of the factor $\gamma$. 

We assume a LiDAR is employed to scan a large-scale open environment free of obstacles at a static pose. When determining the free space by projecting to the depth image $\mathcal{M}_I$ with the standard resolution of $\psi_I$, the resultant occupancy map should show complete free space in a 3D spherical domain, the volume of which is denoted as $\boldsymbol{V}_I$. However, when relaxing the depth image resolution to $\psi=\gamma\psi_I$ for a relaxed depth image $\mathcal{M}$, some free space is falsely determined as unknown when projecting to $\mathcal{M}$ due to the single pixel cases explained in Section~\ref{subsec:determination}. We denote the volume of this case as $\boldsymbol{V}$. An illustration of $\boldsymbol{V}_I$ and $\boldsymbol{V}$ is presented in Fig. 1 in the supplementary materials\cite{cai2023supplementary}. The accuracy function $f(\gamma)$ is defined as the ratio of space correctly determined as free:
    \begin{equation}\label{eq:fdefine}
        f(\gamma) = \frac{\boldsymbol{V}}{\boldsymbol{V}_I}\\
    \end{equation}

We made the following assumption to ease the theoretical analysis:
\begin{assumption}\label{asm:3D}
We assume that the LiDAR is an idealized representation, which assumes a horizontal FoV of \SI{360}{\degree} and a symmetric vertical FoV in the range of $[-\alpha,\alpha]$ with infinitely small angular resolution (e.g., $\psi_{\mathtt{lidar}}\to 0$) and a detection range $R$. The occupancy mapping resolution is $d$. 
\end{assumption}

\begin{theorem}\label{theo:3D}
    The function $f(\gamma)$ is derived based on Assumption \ref{asm:3D} as follows:
	\begin{equation}\label{eq:f}
	f(\gamma) = \begin{cases}
            \frac{\gamma_0^3}{\gamma} + A(1-\frac{\gamma_0}{\gamma})\frac{1}{\gamma^2} & \gamma > \gamma_0\\
            1 & 0<\gamma \leqslant \gamma_0\\
            \end{cases}
	\end{equation}
	where $A$ and $\gamma_0$ are constant factors given by  
	\begin{subequations}
	\begin{align}
	A\!&=\! \frac{3(1-\cos(\alpha))+\frac{12}{\pi}\sin(\alpha)}{1\!-\!\frac{1}{2}(\sin(\alpha)\cos^2(\alpha)\!+\!(1\!-\!\sin(\alpha))^2(2\!+\!\sin(\alpha)))}\\
	\gamma_0\!&=\!\sqrt{3}\sin(\min\{\atan(\frac{1}{\sqrt{2}})\!+\!\alpha,\frac{\pi}{2}\}).		
	\end{align}
	\end{subequations}
\end{theorem} 
\begin{figure}[t]
	\setlength\abovecaptionskip{-0.1\baselineskip}
	\centering
	\includegraphics[width=\linewidth]{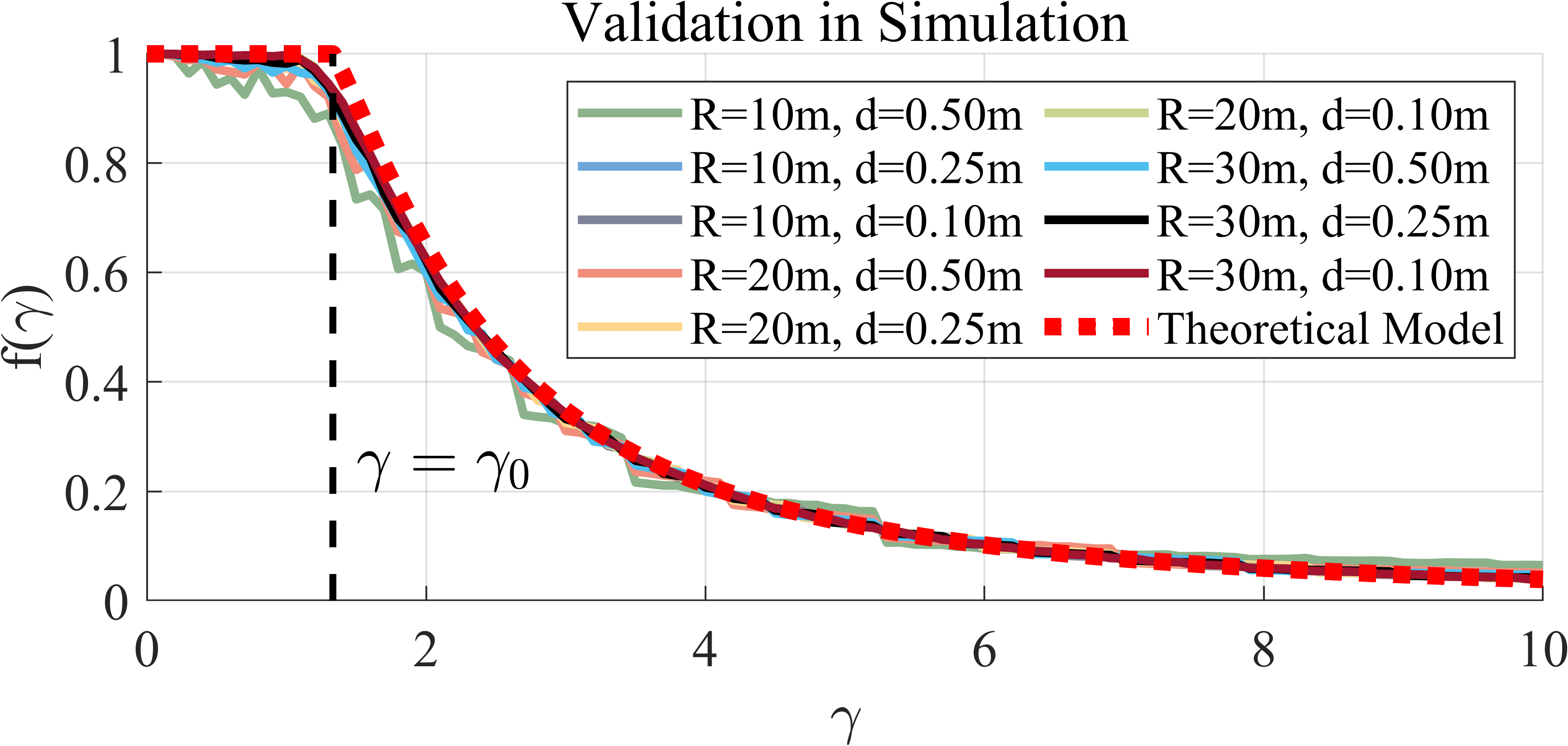}
	\caption{The simulation results for validation on accuracy function $f(\gamma)$.}
	\label{fig:AccAnalysis} 
 \vspace{-0.2cm}
\end{figure} 
\begin{proof}
See Supplementary I\cite{cai2023supplementary}.
\end{proof}

According to Theorem \ref{theo:3D}, the relaxed factor $\gamma$ can be determined by an expected accuracy value {$\omega\in(0,1]$} using the trigonometric form of Cardano's Formula\cite{spiegel2009formula}, given as:
\begin{equation}\label{eq:acc}
	\gamma = \begin{cases}
            2\sqrt{\frac{A}{3\omega}}\cos(\beta) & 0<\omega<1\\
            1 & \omega = 1
        \end{cases}
\end{equation}
where $\beta$ is given as
\begin{equation}
	\beta = \frac{1}{3}\arccos\big(-\frac{A\gamma_0-\gamma_0^3}{2\omega}(\frac{A}{3\omega})^{\frac{3}{2}}\big)
\end{equation}
It is worth noting that for  $\alpha \in [0,\pi/2]$, $\gamma_0$ always satisfies $\gamma_0 \geqslant 1$. Therefore, setting $\gamma = 1$ always results in $f(\gamma) = 1$, regardless of the value of $\alpha$. Besides, it is noted the accuracy function $f(\gamma)$ attains a value of 1 for $\gamma\in(0,\gamma_0]$, as shown in equation (\ref{eq:f}). Therefore, in equation (\ref{eq:acc}), we select $\gamma=1$ as the solution for $\omega=1$.

The accuracy function $f(\gamma)$ is validated through a series of simulation experiments involving various LiDAR detection ranges and map resolutions. Specifically, given a detection range $R$ and map resolution $d$, a batch of point clouds covering the free space within the detection range is generated on a sphere of radius $R$. To compute the volume $\boldsymbol{V}_I$, we count the number of cells in the map with resolution $d$ that are identified as free by projecting the map to the original depth image $\mathcal{M}_I$ and multiplying the result by the unit volume $d^3$. Similarly, we obtain the volume $\boldsymbol{V}$ by projecting the map onto the relaxed depth image $\mathcal{M}$. The accuracy function value $f(\gamma)$ is calculated using the definition in (\ref{eq:fdefine}). In the simulations, we set the LiDAR's FoV $\Phi$ to $[$\SI{-15}{\degree}, \SI{15}{\degree}$]$.

The results of the simulation experiments are presented in Fig. \ref{fig:AccAnalysis}.  The theoretical model fits the experimental curves well. However, slight errors arise when $\gamma$ approaches $\gamma_0$ due to discretization in counting the number of cells to evaluate the volume of covered regions. These errors are evident in the curves of low resolution and short detection range (e.g., in the setup of $R=10\mathrm{m}$ and $d=0.5\mathrm{m}$, which is the green solid line in Fig. \ref{fig:AccAnalysis}).

\section{Occupancy Mapping}\label{sec:mapping}
In this section, we describe how to update and query the occupancy states using the map structure in D-Map.
\subsection{Occupancy Map Structure}\label{subsec:map}
With an aim to optimize the balance between computation and memory efficiency, D-Map leverages a hashing grid map to maintain occupied space and an octree to maintain unknown space. 
\subsubsection{Hashing Grid Map}
As there are typically fewer occupied regions in the environment than free and unknown ones, we maintain the occupied space of the environment in a hashing grid map using the voxel hashing technique\cite{niessner2013hashmap}, which allows efficient update and query operations in a time complexity of $\boldsymbol{\mathcal{O}}(1)$. The hash key value for a given point $\boldsymbol{p} = [x,y,z]$ is computed by a hashing function $\mathtt{Hash}$, defined as follows:
\begin{equation}\label{eq:hash}
	\begin{aligned}
		&\mathtt{Hash}(\boldsymbol{p}) = (\mathtt{P}^2 n_z + \mathtt{P} n_y + n_x)~ \mathtt{mod}~ \mathtt{Q}\\
		&n_x = \lfloor \frac{x}{d}\rfloor,\, n_y = \lfloor \frac{y}{d}\rfloor,\,n_z = \lfloor \frac{z}{d}\rfloor,\\
	\end{aligned}
\end{equation}
where $d$ represents the resolution of the hashing grid map. $\mathtt{P}$ and $\mathtt{Q}$ are large prime numbers while $\mathtt{Q}$ also serves as the size of the hashing table. \texttt{mod} is the modulus calculation between two integers. The value of $\mathtt{P}$ and $\mathtt{Q}$ are carefully selected to minimize conflict probability \cite{teschner2003optimized}, with $\mathtt{P}$ and $\mathtt{Q}$ set to $116101$ and $201326611$ in our work, respectively. 

\subsubsection{Octree}
The tree-based map structure is a commonly used approach for occupancy mapping due to its high memory efficiency. Among various tree-based data structures, the octree  \cite{meagher1982octree} stands out as it outperforms other spatial data structures for dynamic updates \cite{FASTLIO2}. Additionally, the spatial division on the octree naturally allows for determining the occupancy states of large cells during tree updates. Therefore, we utilize an octree to organize the unknown space. 

In D-Map, a node on the octree contains the following elements:
\begin{itemize}
	\item The point array $\mathtt{ChildNodes[8]}$ contains the address to its eight child nodes. The point array is empty if the node is a leaf node.
	\item The center $\mathtt{C}$ of the cell represented by the node.
	\item The size $\mathtt{L}$ of the cell represented by the node.
\end{itemize}
\subsubsection{Initialization}
To initialize the mapping procedure in D-Map, the occupancy states of the environment are considered to be entirely unknown. In the implementation, given an initial bounding box $\mathtt{C_{bbx}}$ of the interested regions, the root node of the octree is initialized to represent an unknown cube $\mathtt{C_{root}}$ whose the center $\mathtt{C}$ is allocated at the center of $\mathtt{C_{bbx}}$. The size $\mathtt{L}$ of the root node is assigned as the longest side length of the bounding box $\mathtt{C_{bbx}}$, and the point array $\mathtt{ChildNodes}$ for child nodes is initialized as empty. Besides, the hash table in the hashing grid map is initialized as an empty table. Notably, the initial bounding box does not restrict the mapping area, as both the octree and the hashing grid map allow for the growth of the mapping space on-the-fly \cite{duberg2020ufomap, niessner2013hashmap}.
\subsection{Occupancy Update}\label{subsec:update}
\begin{algorithm}[t]
	\SetAlgoLined
	\SetKwInOut{Input}{Input}
	\SetKwInOut{Output}{Output}
	\SetKwInOut{Param}{Params}	
 	\SetKwFunction{rasterize}{$\mathtt{Rasterization}$}
	\SetKwFunction{build}{$\mathtt{Build2DSegTree}$}
        \SetKwFunction{addocc}{$\mathtt{UpdateGridMap}$}
        \SetKwFunction{area}{$\mathtt{SensingArea}$}
	\SetKwFunction{update}{$\mathtt{UpdateOctree}$}
	\SetKwFunction{return}{$\mathtt{return}$}
	\SetKwFunction{cube}{$\mathtt{GetCenterAndSize}$}
	\SetKwFunction{intersect}{$\mathtt{Intersected}$}
	\SetKwFunction{contain}{$\mathtt{contain}$}
	\SetKwFunction{delete}{$\mathtt{DeleteNode}$}
	\SetKwFunction{determine}{$\mathtt{DetermineOccupancy}$}
	\SetKwFunction{leaf}{$\mathtt{IsLeaf}$}
	\SetKwFunction{split}{$\mathtt{Split}$}
	\Param{Map Resolution $d$,~Initial Grid Size $\mathtt{E}$,\\ 
                LiDAR Detection Range $R$,\\LiDAR FoV $\Phi$ and $\Theta$\\ }
	\Input{Sensor Pose $\mathtt{T}$,~Point Cloud $\mathcal{P}$ \\}
	\SetKwProg{Fn}{Function}{}{}
	\SetKwProg{Alg}{Algorithm Start}{}{}
	\Alg{}{
            \addocc{$\mathcal{P}$};\\ 
            $\mathcal{M}\leftarrow$\rasterize{$\mathcal{P}$};\\
            $\mathcal{T}\leftarrow$\build{$\mathcal{M}$};\\
            $\mathcal{V}\leftarrow$\area{$\mathtt{T}$,$R$,$\Phi$,$\Theta$};\\
		\update{$\mathtt{RootNode}$,$\mathtt{T}$,$\mathcal{V}$,$\mathcal{T}$};\\
	}
	        \textbf{Algorithm End}\\
	\Fn{\update{$\mathtt{Node}$,$\mathtt{T}$,$\mathcal{V}$,$\mathcal{T}$}}{
		$\mathtt{L}$,$\mathtt{C}\leftarrow$\cube{$\mathtt{Node}$};\\
		\lIf{\texttt{!}\intersect{$\mathtt{T}$,$\mathtt{C}$,$\mathcal{V}$}}{\return}
		\eIf{$\mathtt{L}>\mathtt{E}$}{
			$\mathtt{S}\leftarrow\mathtt{Undetermined}$;
		}{
			$\mathtt{S}\leftarrow$\determine{$\mathtt{T}$,$\mathtt{C}$,$\mathtt{L}$,$\mathcal{T}$};
		}
		\Switch{$\mathtt{S}$}{
			\Case{$\mathtt{Unknown}$}{\return;}
			\Case{$\mathtt{Known}$}{\delete{$\mathtt{Node}$};}
			\Case{$\mathtt{Undetermined}$}{
                    \eIf{$\mathtt{L}\leqslant d$}{
                        \delete{$\mathtt{Node}$}
                    }{
 				\lIf{\leaf{$\mathtt{Node}$}}{\split{$\mathtt{Node}$}}				
				\ForEach{$\mathtt{Child}$ $\mathrm{in}$ $\mathtt{ChildNodes}$}{
					\update{$\mathtt{Child}$,$\mathtt{T}$,$\mathcal{V}$};
				}                   
                    }
			}
		}
		
	}
	\textbf{End Function}
	
	\caption{Occupancy Update}
	\label{alg:update}
    \setlength{\textfloatsep}{0pt} 
\end{algorithm}

The occupancy updates in D-Map involve updating the hashing grid map and the octree, as described in Alg.~\ref{alg:update}. The point clouds are directly inserted into the hashing grid map using the hash function in (\ref{eq:hash}) (Line 2) and rasterized into a depth image $\mathcal{M}$ (Line 3). A 2-D segment tree is constructed from the depth image $\mathcal{M}$ for the subsequent occupancy state determination as described in Section~\ref{sec:method} (Line 4). The sensing area $\mathcal{V}$ for cell extraction is obtained from the sensor pose $\mathtt{T}$, LiDAR detection range $R$ and LiDAR FoV $\Phi\times\Theta$ (Line 5). D-Map employs an on-tree update strategy to update the occupancy states of octree nodes within the LiDAR sensing area $\mathcal{V}$ using a function named $\mathtt{UpdateOctree}$ (Line 6). 

The function $\mathtt{UpdateOctree}$ is described in Lines 9$\sim$33 and explained as follows. Starting from the root node of the octree, the corresponding cell center $\mathtt{C}$ and its size $\mathtt{L}$ are obtained (Line 9). The algorithm checks if the corresponding cell intersects with LiDAR sensing area $\mathcal{V}$, terminating further updates if there is no intersection (Line 10). To avoid meaningless occupancy state determination on a too-large cell that always returns undetermined, we directly split those cells larger than an initial cell size $\mathtt{E}$ for efficiency (Line 11$\sim$13). Otherwise, the occupancy state $\mathtt{S}$ is determined by Alg. \ref{alg:occupancy} (Line 14). The operation on the tree node is decided based on the occupancy state $\mathtt{S}$: If the cell is determined as unknown, the node is kept on the tree, and no further operations are required (Line 17$\sim$19). The node is removed from the tree if the cell is known (i.e., free or occupied) (Line 20$\sim$22). If the occupancy state of the node cannot be determined, further updates are required (Line 23$\sim$32). A special condition is considered when the cell size $\mathtt{L}$ achieves the map resolution $d$ (Line 24$\sim$26). In this case, at least one pixel with a depth no smaller than the cell's minimum depth $\mathtt{BoxMin}$ exists (otherwise, it is determined as unknown for $\mathtt{dMax}<\mathtt{BoxMin}$, see Alg.~\ref{alg:occupancy}). As a result, the cell can be determined as known since it must have been either hit or traversed by the corresponding point cloud on that pixel. Otherwise, if the cell size is greater than the map resolution, we split the node, if not have done so (Line 27), and performed updates on its eight children (Line 28$\sim$30). 

The recursive function $\mathtt{UpdateOctree}$ operates in a manner that visits cells from the largest to the smallest size. This approach allows the determination of occupancy states directly on large cells (if they are unknown or free), thus avoiding unnecessary updates on smaller cells. As a result, occupancy updates on the octree are performed concurrently with the tree traversal, resulting in a so-called  ``on-tree'' update strategy. This strategy is distinct from the classical pipeline \cite{hornung2013octomap,kwon2019superray}, which updates cells at the constant map resolution $d$ after ray-casting. Furthermore, D-Map continuously removes known cells from the environment's unknown space, making it a decremental mapping method. This property enhances the efficiency as the mapping process progresses.

\subsection{Occupancy State Query}\label{subsec:query}
Since two data structures are used in D-Map, two steps are performed to obtain the occupancy state of a cell at map resolution. Firstly, the cell is determined as unknown if its corresponding node exists on the octree. Otherwise, the region is determined as known. Subsequently, the hashing grid map is queried to determine whether it is occupied. If not, the region is determined as free. 

\section{Time Complexity Analysis}\label{sec:timeanalysis}
\begin{figure}[t]
	\setlength\abovecaptionskip{-0.1\baselineskip}
	\centering
	\includegraphics[width=\linewidth]{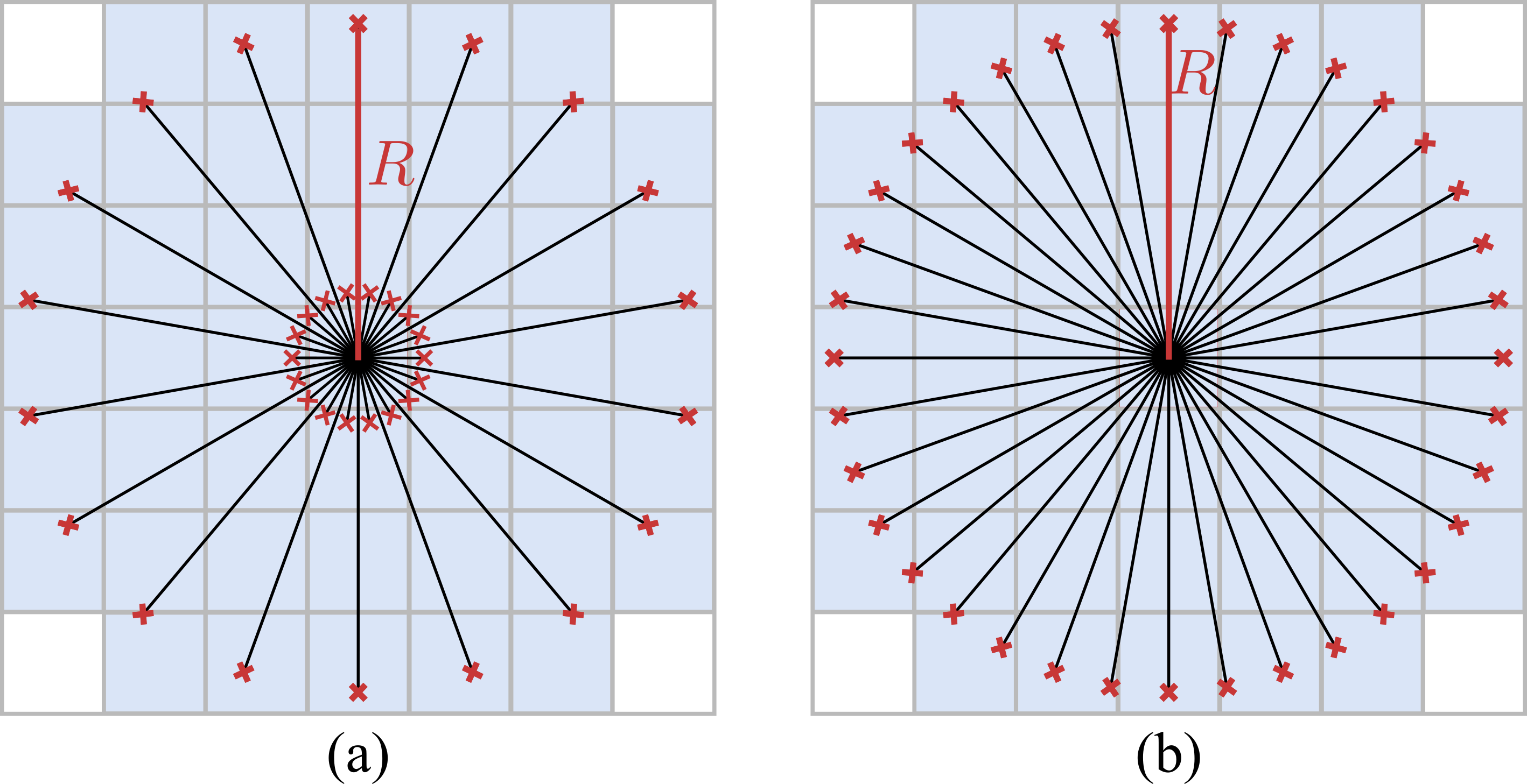}
	\caption{A 2-D illustration of the worst-case occupancy updates for (a) D-Map and (b) ray-casting-based methods. In the worst-case scenario for D-Map, two neighboring points appear at the minimum and maximum distance to the sensor origin, resulting in a serrated-shaped point cloud. In contrast, the worst-case scenario for ray-casting-based methods results in a spherical-shaped point cloud located at the maximum distance to the sensor origin. }
	\label{fig:worstcase} 
 \vspace{-0.3cm}
\end{figure} 
In this section, we analyze the time complexity of updating and querying the occupancy states on D-Map. To facilitate further analysis of the benchmark experiments in Section~\ref{sec:benchmark}, we additionally provide a time complexity analysis of the classical occupancy update pipeline, which involves ray-casting prior to map updates. 

\subsection{Occupancy State Update}

The occupancy state updates in D-Map are composed of two parts: updating occupied cells on the hashing grid map and updating unknown cells on the octree. The time complexity for occupancy updates is mainly due to the on-tree update in the octree data structure, as the time of maintaining occupied points in a hashing grid map is constant at $\boldsymbol{\mathcal{O}}(1)$. The on-tree update strategy involves searching for nodes within the LiDAR sensing area on the octree and determining their occupancy states. To analyze the time complexity, we first consider the number of nodes to be visited on the octree, as presented in the following lemma:
\begin{lemma}\label{lemma:nodes}
    Suppose the maximum side length of environment bounding box $C_{bbx}$ on the octree is $D$, and map resolution is $d$. The number of nodes to be visited $\mathcal{S}_{\mathtt{DMap}}$ on the octree for the on-tree update strategy in D-Map is bounded as:
    \begin{equation}
        \begin{aligned}
            \mathcal{S}_{\mathtt{DMap}}=\boldsymbol{\mathcal{O}}\big(n\frac{R}{d}\log{(\frac{D}{d})}\big)
        \end{aligned}
    \end{equation}
    where $n$ is the number of points projected to the depth image, and $R$ is the detection range of the LiDAR sensor.
\end{lemma}
\begin{proof}
See Supplementary II-A\cite{cai2023supplementary}.
\end{proof}

Then we derive the time complexity of occupancy state determination (see Alg. \ref{alg:occupancy}) that is frequently used in the update process. 
\begin{lemma}\label{lemma:determination}
The time complexity of occupancy state determination is $\boldsymbol{\mathcal{O}}(\log({\Phi_I}/{\psi_I}) \log({\Theta_I}/{\psi_I}))$, where $\Phi_I$ and $\Theta_I$ are the FoV angles of the depth image and $\psi_I$ is the depth image resolution.
\end{lemma}
\begin{proof}
See Supplementary II-B\cite{cai2023supplementary}.
\end{proof}

Finally, the time complexity for occupancy updates on D-Map is given as follows:
\begin{theorem}\label{theo:update}
    The time complexity of occupancy updates in D-Map is  $\boldsymbol{\mathcal{O}}\big(\frac{nR}{d}\log{(\frac{D}{d})}\big)$, where $n$ is the number of points projected to the depth image, $d$ is the map resolution, and $R$ is the LiDAR detection range.
\end{theorem}
\begin{proof}
See Supplementary II-C\cite{cai2023supplementary}.
\end{proof}

We provide the time complexity of the occupancy updates on a grid-based map and an octree-based map as follows:
\begin{theorem}\label{theo:classic}
    The time complexity to update the occupancy states on a grid-based map is $\boldsymbol{\mathcal{O}}(\frac{nR}{d})$ and on an octree-based map is $\boldsymbol{\mathcal{O}}(\frac{nR}{d}\log(\frac{D}{d}))$.
\end{theorem}
\begin{proof}
See Supplementary II-D\cite{cai2023supplementary}.
\end{proof}

\begin{remark} \rm{Although D-Map and the octree-based map have the same time complexity of $\boldsymbol{\mathcal{O}}(\frac{nR}{d}\log(\frac{D}{d}))$, their sources of time complexity are entirely different. Since D-Map updates the occupancy states concurrently with the tree traversal, as discussed in Section~\ref{subsec:update}, the time complexity of D-Map depends solely on the number of visited nodes $\mathcal{S}_{\mathtt{DMap}}$ during the on-tree update process, which contributes to $\boldsymbol{\mathcal{O}}(\frac{nR}{d}\log(\frac{D}{d}))$ directly. In contrast, occupancy updates on an octree-based map consist of two consecutive procedures: ray-casting, which contributes to $\boldsymbol{\mathcal{O}}(\frac{nR}{d})$, and map updates, which contribute to $\boldsymbol{\mathcal{O}}(\log(D/d))$, resulting in the same total time complexity of $\boldsymbol{\mathcal{O}}(\frac{nR}{d}\log(\frac{D}{d}))$.}

\end{remark}
\begin{remark}
    \rm{The time complexity of D-Map is calculated without considering the removal of known cells, which could reduce the number of visited nodes $\mathcal{S}_{\mathtt{DMap}}$ substantially. In contrast, existing ray-casting-based methods (either grid-based map or octree-based map) do not have such decremental properties. }
\end{remark}
\begin{remark}
    \rm{We derive the time complexity of occupancy updates for our D-Map and ray-casting-based methods (i.e., the grid-based and octree-based maps) based on their respective worst-case scenario, as illustrated in Fig.~\ref{fig:worstcase}. However, the occurring chances of their worst-case scenarios are different. Generally speaking, the worst-case scenario for ray-casting-based methods is more likely to occur in the real world (e.g., large open space) than for our D-Map.}
\end{remark}

\subsection{Occupancy State Query}
The time complexity of the occupancy state query in D-Map is provided as follows:
\begin{theorem}\label{theo:query}
    The time complexity of querying the occupancy state of a cell in D-Map is $\boldsymbol{\mathcal{O}}(\log(D/d))$.
\end{theorem}
\begin{proof}
See Supplementary II-E\cite{cai2023supplementary}.
\end{proof}
\section{Benchmark Results}\label{sec:benchmark}
In this section, exhaustive benchmark experiments are conducted on various datasets with different LiDARs to evaluate the computational efficiency, accuracy, and memory consumption against other state-of-the-art occupancy mapping approaches. Besides the time performance evaluation, we conduct an ablation study based on the time complexity to provide deeper insight into the efficiency of our approach. 
\subsection{Datasets}
\begin{table}[t]
	\setlength{\tabcolsep}{6.7pt}
	\centering
	\caption{Details of Datasets for Benchmark Experiment}
	\label{tab:dataset}	
	\begin{threeparttable}
		\begin{tabular}{@{}cccc@{}}
			\toprule
			& Mapped Area            & Number of & Average Point Number \\
			& ($\mathrm{m}^3$)                & Scans    & (per scan)           \\ \midrule
			\textit{FR\_079}      & $48\times37\times5$    & 66       & 89445.9              \\
			\textit{Freiburg}     & $293\times168\times29$ & 81       & 247817.1             \\
			\textit{Workshop}     & $40\times40\times8$    & 7340     & 23540.8              \\
			\textit{MainBuilding} & $636\times251\times26$ & 1402     & 22124.9              \\
			\textit{Kitti\_04}    & $551\times160\times37$ & 271      & 125265.2             \\
			\textit{Kitti\_06}    & $616\times181\times47$ & 1101     & 122189.4             \\
			\textit{Kitti\_07}    & $362\times347\times40$ & 1101     & 121225.5             \\ \bottomrule
		\end{tabular}
	\end{threeparttable}
 \vspace{-0.2cm}
\end{table}
The experiments are conducted on three open datasets and one private dataset. The first public dataset is \textit{Kitti} dataset which captured depth measurements by a Velodyne HDL-64E rotating 3D laser scanner at \SI{10}{\hertz} \cite{geiger2013kitti}. Considering the tremendous memory consumption associated with grid-based maps that will be utilized in our benchmark experiments, we chose sequences \textit{Kitti\_04}, \textit{Kitti\_06}, and \textit{Kitti\_07} to conduct the benchmark evaluation. The second public dataset is an outdoor sequence \textit{MainBuilding} provided in FAST-LIO \cite{xu2021fastlio}, which used a semi-solid state 3D LiDAR sensor Livox Avia to collect data at \SI{10}{\hertz}. The third public dataset, provided in the work Octomap\cite{hornung2013octomap}, includes an indoor sequence \textit{FR-079} and an outdoor sequence \textit{Freiburg}. In addition, we evaluate the benefit of decremental property on a private dataset \textit{Workshop}, where a cluttered indoor environment is completely scanned manually using 3D LiDAR Livox Avia at \SI{10}{\hertz}. The reconstruction result of \textit{Workshop} dataset is available in \cite{kong2023marsim}. It is noticed that the odometry estimation for occupancy mapping in sequence \textit{Workshop} and \textit{MainBuilding} was acquired by our LiDAR-inertial odometry framework FAST-LIO2\cite{FASTLIO2}. Table \ref{tab:dataset} provides further details on the aforementioned datasets.
\subsection{Evaluation Setup}
We compared D-Map with state-of-the-art occupancy mapping methods commonly used in robotics applications, {including a grid-based method (Grid Map updated by 3DDDA\cite{1987raycast}), a tree-based method (Octomap\cite{hornung2013octomap}), and the extended versions of Grid Map and Octomap using super rays and culling regions\cite{kwon2019superray} (denoted with ``SR\&CR(Grid)" and ``SR\&CR(Octo)")}. These occupancy mapping methods were selected for their efficient updates and widespread usage in the field. For the Octomap, SR\&CR(Grid), and SR\&CR(Octo), we used their open-source implementations available on GitHub repositories\footnote{https://github.com/OctoMap/octomap}\footnote{https://github.com/PinocchioYS/SuperRay}. For the grid map, we modify the open-source SR\&CR(Grid) by disabling the super rays and culling regions and termed it as ``GridMap''. It is noticed that all grid-based maps employed voxel hashing to achieve better memory efficiency. The benchmark experiments are conducted at various map resolutions, ranging from a rough resolution of \SI{1.0}{\meter} to a finer resolution of \SI{1}{\cm} for indoor sequences. Due to system memory limitations, the minimum map resolution for outdoor sequences is set to \SI{5}{\cm}. 

The computation platform for evaluation is an Intel NUC computer with CPU Intel i7-10710U and \SI{64}{\giga\byte} RAM. To facilitate evaluation on high-resolution maps in large-scale environments, a \SI{64}{\giga\byte} swap space on Solid State Drive (SSD) is allocated to account for extensive memory usage. 

D-Map uses a fixed set of parameters for all benchmark experiments. The completeness threshold $\mathtt{\varepsilon}$ in Alg. \ref{alg:occupancy} is set to $0.8$, while the initial cell size $\mathtt{E}$ in Alg. \ref{alg:update} is set to \SI{5.0}{\m}. The depth image resolution is determined using (\ref{eq:res}) without relaxation {(i.e., $\gamma=1$, which is not greater than $\gamma_0$ and leads to $f(\gamma)=1$, as explained in Section~\ref{subsec:acc_analysis})}. The LiDAR detection range $R$ and LiDAR angular resolution $\psi_{\mathtt{lidar}}$ are acquired from LiDARs' manual sheets, except for \textit{Workshop} indoor sequence whose detection range is assigned as \SI{60}{\meter} in consideration of the indoor environment.

Given the high accuracy and low false alarm rate of LiDAR sensors, the space that is hit by a LiDAR pulse can be reliably considered occupied, and that passed by the laser ray can be considered free. We use the following parameter setup for other occupancy mapping approaches in the benchmark experiment to ensure consistency with this assumption that our D-Map has utilized. We set the probabilities for hit and miss by a ray (i.e., for occupied and free space) to $P_{\mathtt{occ}} = 0.9999$ and $P_{\mathtt{free}}=0.4999$, respectively, and the minimum and maximum clamping probabilities to $P_{\mathtt{min}} = 0.499$ and $P_{\mathtt{max}} = 0.9999$. Moreover, we set the initial occupancy probabilities of all cells to 0.5, which represents unknown cells without any update (or observation) from LiDAR points, and we regard the cells with occupancy probability values smaller than $0.5$ (due to subsequent LiDAR points update) as free and those with occupancy probability values larger than $0.5$ as occupied. Finally, the methods SR\&CR(Grid) and SR\&CR(Octo) both require a threshold parameter $k$ for the minimum number to generate super rays, which is set to the default value $k=20$ as in \cite{kwon2019superray}.

\subsection{Efficiency Evaluation and Analysis}
\subsubsection{Benchmark Results}
\begin{table*}[t]
	
	\centering
	\caption{Comparison of Update Time (ms) at Different Resolution}
	\label{tab:eff}
	\setlength{\tabcolsep}{3.5pt}
 \renewcommand*{\arraystretch}{1.0}
	\begin{threeparttable}
		\begin{tabular}{@{}cccccccccccccccc@{}}
			\toprule
			\multicolumn{1}{l}{} & \multicolumn{5}{c}{\textit{Kitti\_04}}                                                 & \multicolumn{5}{c}{\textit{Kitti\_06}}                                                & \multicolumn{5}{c}{\textit{Kitti\_07}}                                               \\\cmidrule(l){2-6}  \cmidrule(l){7-11} \cmidrule(l){12-16}
			Resolution (m)       & 1.0            & 0.5            & 0.25            & 0.1             & 0.05             & 1.0            & 0.5            & 0.25           & 0.1             & 0.05             & 1.0            & 0.5            & 0.25           & 0.1             & 0.05            \\ \midrule
			GridMap                & 47.32          & 89.76          & 190.78          & 720.33          & 3636.54          & 50.22          & 97.70          & 203.56         & 1110.29         & 10505.13         & 37.04          & 68.64          & 140.28         & 536.51          & 3431.81         \\
			Octomap              & 398.52         & 637.52         & 964.52          & 2459.13         & 7623.91          & 437.43         & 664.70         & 1076.62        & 2629.29         & 6558.01          & 299.72         & 469.21         & 775.76         & 1777.65         & 3877.65         \\
			SR\&CR(Grid)          & 43.08          & 129.51         & 298.76          & 1398.72         & 7360.95          & 48.86          & 116.96         & 322.30         & 2422.06         & 17977.50         & 30.17          & 72.91          & 180.42         & 995.12          & 6290.95         \\
			SR\&CR(Octo)           & 44.47          & 125.31         & 389.37          & 2330.39         & 9287.51          & 56.49          & 161.41         & 554.35         & 2974.88         & 12361.56         & 30.05          & 85.18          & 307.54         & 1634.40         & 6153.84         \\
			Ours                 & \textbf{22.67} & \textbf{39.05} & \textbf{101.14} & \textbf{470.67} & \textbf{1905.32} & \textbf{22.71} & \textbf{36.08} & \textbf{85.91} & \textbf{367.08} & \textbf{1371.69} & \textbf{19.12} & \textbf{27.27} & \textbf{59.33} & \textbf{209.46} & \textbf{779.32} \\ 
		\end{tabular}
	\end{threeparttable}	
	\setlength{\tabcolsep}{8.2pt}
	\begin{threeparttable}
            \begin{tabular}{@{}ccccccccccccc@{}}
            \toprule
            \multicolumn{1}{l}{} & \multicolumn{7}{c}{\textit{Workshop}}                                                                              & \multicolumn{5}{c}{\textit{MainBuilding}}                                           \\ \cmidrule(l){2-8}  \cmidrule(l){9-13} 
            Resolution (m)       & 1.0           & 0.5           & 0.25          & 0.1            & 0.05           & 0.025          & 0.01            & 1.0           & 0.5            & 0.25           & 0.1             & 0.05            \\ \midrule
            GridMap             & 4.94          & 8.42          & 15.97         & 50.17          & 127.43         & 346.68         & 13702.461       & \textbf{9.63} & \textbf{19.91} & \textbf{45.88} & \textbf{162.65} & \textbf{640.83} \\
            Octomap              & 42.21         & 59.62         & 92.76         & 175.17         & 297.93         & 589.07         & 7568.64         & 64.82         & 106.43         & 199.05         & 529.66          & 1333.22         \\
            SR\&CR(Grid)      & \textbf{2.64} & 7.35          & 18.65         & 72.68          & 244.95         & 1514.41        & 39624.96        & 10.86         & 27.65          & 72.61          & 315.12          & 1384.12         \\
            SR\&CR(Octo)       & \textbf{2.64} & 9.29          & 29.96         & 162.54         & 717.84         & 3352.30        & 50174.95        & 12.61         & 44.12          & 156.80         & 908.08          & 3783.49         \\
            Ours                 & 3.02          & \textbf{4.45} & \textbf{9.26} & \textbf{20.10} & \textbf{42.12} & \textbf{88.16} & \textbf{915.20} & 11.16         & 31.56          & 112.52         & 332.83          & 991.66          \\
            \end{tabular}
	\end{threeparttable}	
	\begin{threeparttable}
		\begin{tabular}{@{}ccccccccccccc@{}}
			\toprule
			\multicolumn{1}{l}{} & \multicolumn{7}{c}{\textit{FR\_079}}                                                                                 & \multicolumn{5}{c}{\textit{Freibrug}}                                                 \\ \cmidrule(l){2-8}  \cmidrule(l){9-13} 
			Resolution (m)       & 1.0           & 0.5           & 0.25           & 0.1            & 0.05           & 0.025           & 0.01            & 1.0            & 0.5            & 0.25           & 0.1             & 0.05             \\ \midrule
			GridMap             & 10.14         & 15.47         & 26.47          & 63.77          & 152.04         & 438.84          & 4082.14         & 61.41          & 109.99         & 230.90         & 1021.84         & 6005.56          \\
			Octomap              & 3.89          & 7.90          & 17.45          & 66.48          & 347.57         & 1514.97         & 7568.64         & 356.43         & 548.27         & 990.92         & 3357.36         & 10637.87         \\
			SR\&CR(Grid)      & \textbf{3.76} & 7.91          & 16.34          & 53.34          & 237.25         & 816.58          & 4805.09         & 40.72          & 110.27         & 285.78         & 1522.21         & 7911.33          \\
			SR\&CR(Octo)       & 3.87          & \textbf{7.70} & 18.28          & 68.21          & 375.14         & 1613.46         & 8402.84         & 43.68          & 112.59         & 354.07         & 2186.24         & 8709.99          \\
			Ours                 & 10.09         & 11.29         & \textbf{15.11} & \textbf{29.01} & \textbf{65.22} & \textbf{205.06} & \textbf{756.33} & \textbf{31.40} & \textbf{40.48} & \textbf{82.55} & \textbf{678.09} & \textbf{3599.96} \\ \bottomrule
		\end{tabular}
	\end{threeparttable}
 \vspace{-0.3cm}
\end{table*}
We summarize the average update time of each occupancy mapping approach and report the results in Table \ref{tab:eff}. The update time for SR\&CR(Grid) and SR\&CR(Octo) includes the time required for preprocessing, which involves generating super rays and constructing culling regions, as well as the time for occupancy updates involving ray-casting and map updates. In contrast, for GridMap and Octomap, only the time for ray-casting and map updates is considered, as no preprocessing is required. As for our D-Map, we count the preprocessing time for depth image rasterization and 2-D segment tree construction as well as the occupancy update time for updating the octree and hashing grid map. As presented in Table~\ref{tab:eff}, our D-Map consistently outperforms the other mapping approaches across various map resolutions in three sequences of the \textit{Kitti} dataset. Particularly, at a map resolution of \SI{5}{\cm} in \textit{Kitti\_06}, D-Map achieves a remarkable speedup of $7.66$ times and $4.78$ times faster than GridMap and Octomap, respectively. However, in the \textit{MainBuilding} datasets, GridMap shows the highest efficiency at different map resolutions, followed by our D-Map and {SR\&CR(Grid)}. The lower efficiency of D-Map is caused by the sparse clouds in \textit{MainBuilding} sequence, which only has an average number of \SI{22}{k} per scan, while \textit{Kitti} datasets have over \SI{120}{k} points per scan on average. This finding suggests that our method is less efficient in updating sparse point clouds since a sparse depth image rasterized from the sparse point clouds causes redundant occupancy state queries on those unobserved regions during the updating process. In contrast, ray-casting-based mapping approaches exhibit better performance due to their precise traversal of observed regions. However, this disadvantage of D-Map could be compensated by the decremental property as shown in the \textit{Workshop} indoor sequence, where the scanning process frequently visited the previously explored areas with an aim to map the entire indoor environment exhaustively. This scanning process, which is often used in mapping applications, enables our D-Map to continuously remove known space from the octree map, achieving higher efficiency than the other methods. Fig.~\ref{fig:workshop} presents the update time consumption of different occupancy approaches at a \SI{1}{\cm} map resolution, clearly illustrating the decreasing update time of D-Map as the environment is explored.

\begin{figure}[t]
	\setlength\abovecaptionskip{-0.1\baselineskip}
	\centering
	\includegraphics[width=\linewidth]{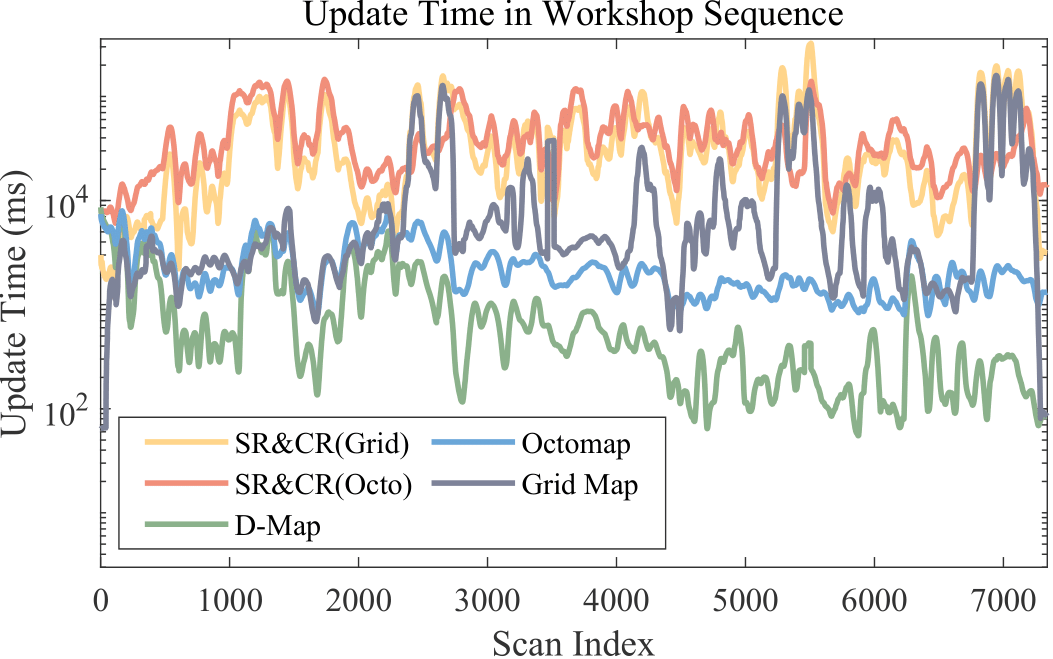}
	\caption{Update time in \textit{Workshop} indoor sequence at a high map resolution of \SI{1}{\cm}. }
	\label{fig:workshop} 
    \vspace{-0.3cm}
\end{figure} 
D-Map demonstrates superior performance compared to other approaches at high map resolutions ($d \leqslant \SI{0.25}{\meter}$) in the indoor sequence \textit{FR\_079}, as well as at various resolutions in the outdoor sequence \textit{Freiburg}. However, at low resolutions ($d \geqslant \SI{0.5}{\meter}$) in \textit{FR\_079} sequence, the performance is limited by the preprocessing time required for depth image rasterization (approximately \SI{5.7}{\ms}) and the processing time to insert occupied points into the hashing grid map (about \SI{4.2}{\ms}).

Across all benchmark experiments conducted in this study, we have observed that SR\&CR(Grid) and SR\&CR(Octo) exhibit low efficiency at high resolution compared to GridMap and Octomap, respectively, without super rays or culling regions. This phenomenon has been previously noted and discussed in \cite{kwon2019superray}, where it is attributed to the low grouping ratio of super rays and the need for the timely rebuilding of culling regions at high resolution. Additionally, we have also observed anomalous time performances when comparing the grid-based method (i.e., GridMap and SR\&CR(Grid)) against the tree-based methods (i.e., Octomap and SR\&CR(Octo)). Normally, grid-based maps are expected to offer superior update efficiency compared to tree-based maps. However, we have observed that grid-based methods are slower than tree-based methods at high resolution ($d=$\SI{5}{\cm}) in the sequences \textit{Kitti\_06} and \textit{Kitti\_07}. The inefficiency of grid-based methods can be attributed to their reliance on swap space on the SSD during the update process since their memory usage exceeds the capacity of RAM. As memory access on SSD is significantly slower than RAM, this reliance on swap space can lead to the observed slower update performance.
\subsubsection{Efficiency Analysis}
We present a comprehensive analysis of our superior efficiency performance based on the time complexity analysis in Section~\ref{sec:timeanalysis}. As demonstrated in Section \ref{sec:timeanalysis}, the efficiency of D-Map is solely determined by the number of visited cells $\mathcal{S}_{\mathtt{DMap}}$ due to its ``on-tree" update strategy. However, GridMap, Octomap, SR\&CR(Grid), and SR\&CR(Octo), which follow the classical occupancy update pipeline using ray-casting, have an efficiency that depends on two factors: the number of cells $\mathcal{S}_{\mathtt{rc}}$ traversed by rays and the time complexity for map updates, which is $\boldsymbol{\mathcal{O}}(1)$ for grid-based methods and $\boldsymbol{\mathcal{O}}(\log(D/d))$  for octree-based methods. Therefore, we conduct a comparison between the number of visited cells $\mathcal{S}_{\mathtt{DMap}}$ in D-Map and the number of updated cells $\mathcal{S}_{\mathtt{rc}}$ in the ray-casting-based methods, followed by a comparison of the corresponding occupancy update time (without counting the preprocessing time). Specifically, we opt to compare D-Map with SR\&CR(Grid) and SR\&CR(Octo), which incorporate super rays and culling region techniques to reduce the number of cells. The comparison is carried out on \textit{FR\_079} and \textit{Freiburg} sequences, with the results presented in Fig.~\ref{fig:grids}.

In comparison between $\mathcal{S}_{\mathtt{DMap}}$ and $\mathcal{S}_{\mathtt{rc}}$, we conduct an ablation study to investigate the performance of all methods (i.e., SR\&CR(Grid), SR\&CR(Octo), and D-Map) without any removal of cells. For our D-Map, we disable the removal of known cells from the octree. For SR\&CR(Grid) and SR\&CR(Octo), we disable the culling region technique and denote these modified methods as ``SR(Grid)'' and ``SR(Octo)'', respectively. As shown in Fig.~\ref{fig:grids}, SR(Grid) exhibits the highest count of cells across all resolutions. followed by SR(Octo) and our D-Map without removal. The difference in the number of updated cells between SR(Grid) and SR(Octo) is attributed to the batching-based method in SR(Octo) that batches the uniform cells traversed by multiple rays to reduce the cells. The difference between the two methods becomes smaller at higher resolution as there are fewer uniform cells traversed by multiple rays. It is worth noticing that, although batching is essential to tree-based methods to account for the cost of tree updates, it is unnecessary for grid-based maps, which have an efficient update in simple time complexity of $\boldsymbol{\mathcal{O}}(1)$. As analyzed in Section~\ref{sec:timeanalysis}, D-Map has a theoretical number of visited nodes $\mathcal{S}_{\mathtt{DMap}}=\boldsymbol{\mathcal{O}}(\frac{nR}{d}\log(\frac{D}{d}))$ in its worst case, which is larger than the theoretical number of traversed cells $\mathcal{S}_\mathtt{rc}=\boldsymbol{\mathcal{O}}(\frac{nR}{d})$ in the worst case of ray-casting-based methods. However, in this ablation study, D-Map without removal visits a comparable number of nodes to SR(Octo) because the worst case of D-Map rarely occurs. Specifically, D-Map without removal processes fewer nodes than SR(Octo) at low resolutions (i.e., $d\geqslant 0.25\mathrm{m}$) in indoor sequence \textit{FR\_079} and at all resolutions except \SI{5}{\cm} in outdoor sequence \textit{Freiburg}. When approaching high resolutions (e.g., $d<0.25\mathrm{m}$ in indoor sequence and $d=0.05\mathrm{m}$ in outdoor sequence), the number of visited nodes $\mathcal{S}_\mathtt{DMap}$ on D-Map is closer to its worst case, which has a higher increasing rate on a logarithmic scale than $\mathcal{S}_\mathtt{rc}$. 

The removal of known cells and the culling region technique are then enabled in corresponding methods for further investigation. D-Map demonstrates the most significant reduction in the number of cells among all approaches when the removal technique is enabled, while the culling region technique only slightly lowers the number of updated cells in SR\&CR(Grid) and SR\&CR(Octo).
\begin{figure*}[t]
    \setlength\abovecaptionskip{-0.1\baselineskip}
	\centering
	\includegraphics[width=\textwidth]{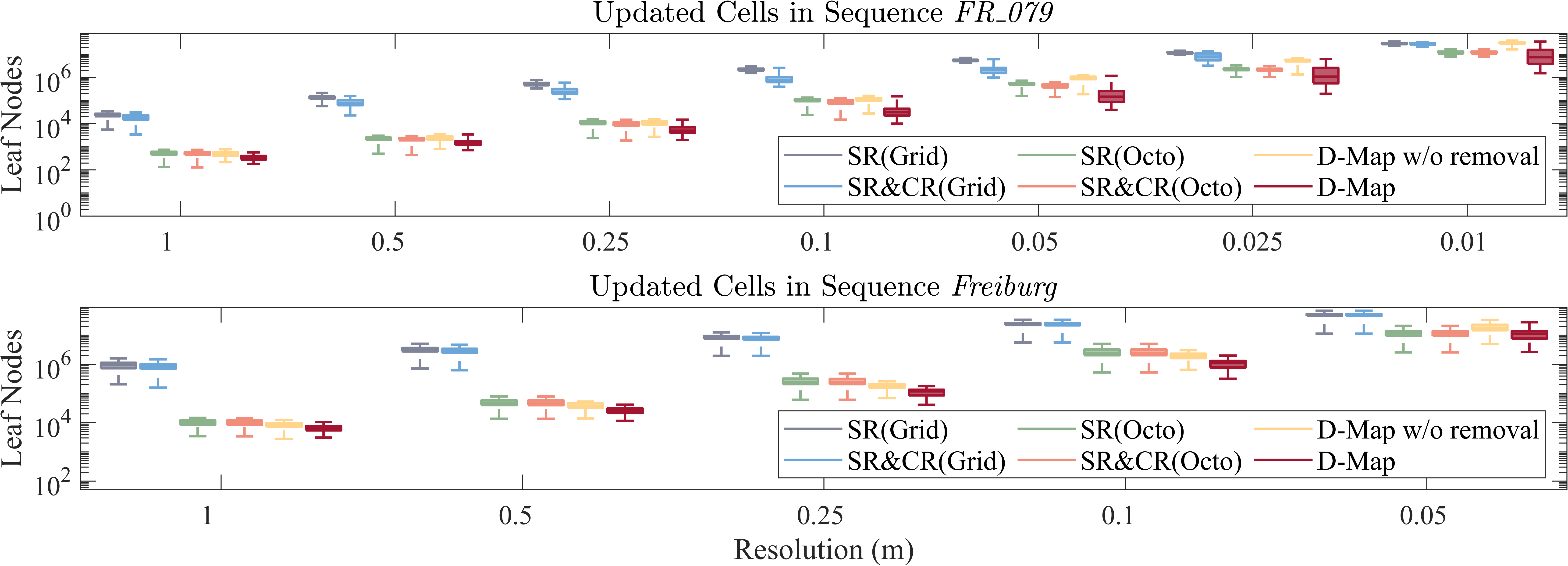}
	\caption{The number of cells to be updated in sequences \textit{FR\_079} and \textit{Freiburg}. }
	\label{fig:grids} 
     \vspace{-0.2cm}
\end{figure*}

\begin{figure*}[t]
    \setlength\abovecaptionskip{-0.1\baselineskip}
	\centering
	\includegraphics[width=\textwidth]{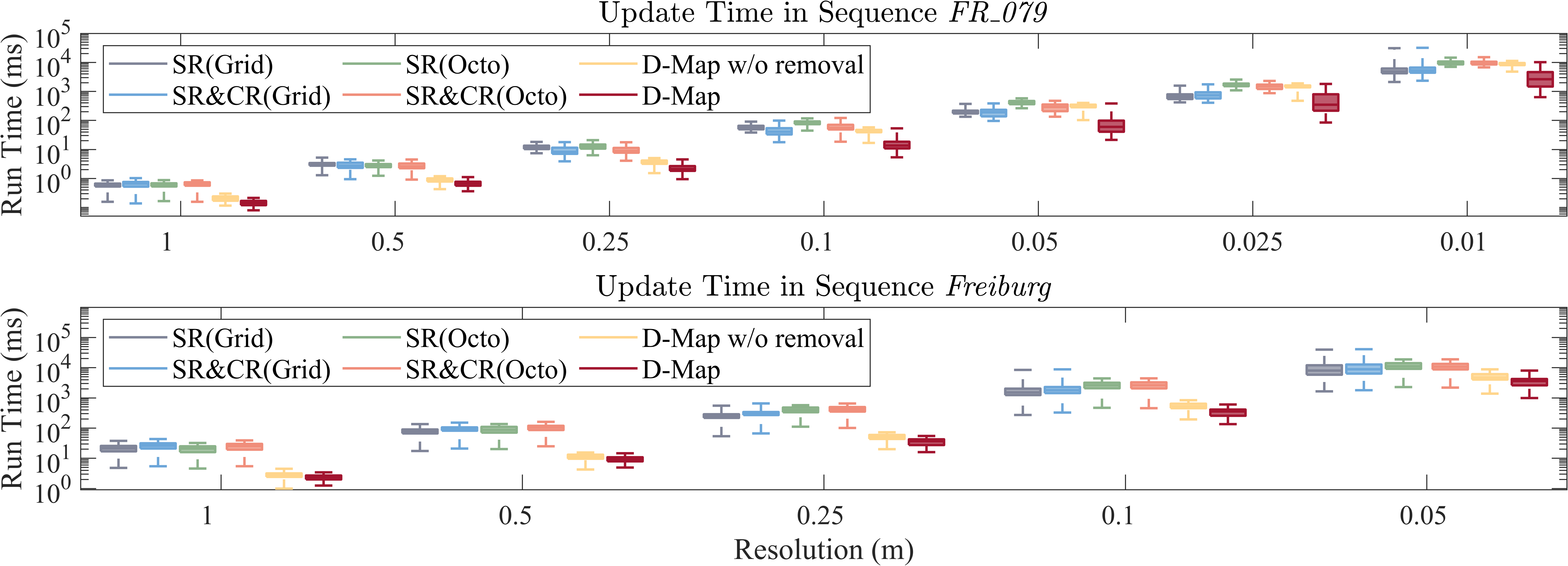}
	\caption{The update time in sequences \textit{FR\_079} and \textit{Freiburg}. }
	\label{fig:updatetime} 
    \vspace{-0.3cm}
\end{figure*}

Finally, we compare the time consumption for occupancy updates in both sequences, as shown in Fig. \ref{fig:updatetime}. When disabling the removal of known cells in D-Map and the culling region technique, the results reveal that D-Map without removal outperforms the others across all map resolutions in sequence \textit{Freiburg}, as well as at low resolutions (i.e., $d>5\mathrm{cm}$) in sequence \textit{FR\_079}. At high resolutions (i.e.,  $d\leqslant0.5\mathrm{m}$) in sequence \textit{FR\_079}, SR(Grid) achieves better performance than D-Map without removal due to a smaller number of cells to be updated. However, despite having a slightly larger number of visited nodes, D-Map without removal maintains its leading position over SR(Octo), which is hindered by the timely updates on the tree structure. When enabling the removal of known cells, D-Map demonstrates the highest efficiency across different map resolutions, owing to the fewest cells requiring to be updated.

\subsection{Accuracy and Memory Evaluation}
\subsubsection{Accuracy Benchmark}
\renewcommand*{\arraystretch}{1.0}
\begin{table*}
	\setlength{\tabcolsep}{4.0pt}
	\caption{Mapping Accuracy at Different Resolution}
	\label{tab:acc}
	\begin{threeparttable}
            \begin{tabular}{@{}ccccccccccccccc@{}}
            \toprule
            \multicolumn{1}{l}{}  & \multicolumn{7}{c}{Unknown Space}                                   & \multicolumn{7}{c}{Free Space}                                      \\ \cmidrule(l){2-8}  \cmidrule(l){9-15}  
            Resolution (m)        & 1.0     & 0.5     & 0.25    & 0.1     & 0.05    & 0.025   & 0.01    & 1.0     & 0.5     & 0.25    & 0.1     & 0.05    & 0.025   & 0.01    \\ \midrule
            \textit{Kitti\_04}    & 97.33\% & 98.29\% & 99.67\% & 99.77\% & 99.46\% & -       & -       & 98.54\% & 98.49\% & 96.73\% & 94.61\% & 93.34\% & -       & -       \\
            \textit{Kitti\_06}    & 97.77\%  & 98.62\% & 99.87\% & 99.84\% & 99.24\% & -       & -       & 98.80\% & 98.77\% & 97.21\% & 94.99\% & 94.53\% &         &         \\
            \textit{Kitti\_07}    & 97.68\% & 99.50\% & 99.64\% & 99.87\% & 99.40\% & -       & -       & 98.40\% & 98.36\% & 96.91\% & 94.93\% & 94.24\% &         &         \\
            \textit{FR\_079}      & 83.53\% & 96.46\% & 97.86\% & 98.96\% & 99.66\% & 99.17\% & 95.39\% & 95.06\% & 95.72\% & 96.55\% & 97.88\% & 95.14\% & 91.68\% & 93.56\% \\
            \textit{Freiburg}     & 97.79\% & 98.93\% & 99.39\% & 99.91\% & 99.96\% & -       & -       & 98.45\% & 98.61\% & 98.80\% & 92.60\% & 83.28\% & -       & -       \\
            \textit{Workshop}     & 98.68\% & 99.19\% & 99.61\% & 99.82\% & 99.87\% & 99.95\% & 99.99\% & 95.11\% & 96.74\% & 98.08\% & 98.80\% & 99.01\% & 98.90\% & 99.00\% \\
            \textit{MainBuilding} & 99.51\% & 99.71\% & 99.77\% & 99.58\% & 99.80\% & -       & -       & 97.65\% & 97.56\% & 97.08\% & 95.10\% & 87.76\% & -       & -       \\ \bottomrule
            \end{tabular}
	\end{threeparttable}
	\centering
 \vspace{-0.3cm}
\end{table*}
We conduct accuracy benchmark experiments using the mapping results from Octomap\cite{hornung2013octomap} as ground truth. Specifically, we calculate the accuracy of D-Map by counting the cells with correct occupancy states over the total number of cells inside the mapping area and with the corresponding states. The benchmark results of unknown space and free space are presented in Table \ref{tab:acc}. It is noted that the accuracy of occupied space is very close to \SI{100}{\percent} in all experiments and thus not presented in the table. This is due to our high confidence in occupied space based on the assumption of high accuracy and low false alarm rate of LiDAR sensors.

The overall accuracy performance of D-Map shows a slight accuracy loss compared to Octomap, as reported in Table \ref{tab:acc}. The inaccuracy mainly arises from the occupancy state determination module on the depth image. At low resolutions (i.e., $d\geqslant0.5\mathrm{m}$), the accuracy of unknown space is relatively lower due to the more severe discretization error introduced in (\ref{eq:area}) when determining the projected area on a lower-resolution depth image. The results also show a decreasing accuracy of free space when the map resolution grows smaller. This decreasing accuracy is introduced by projecting a cell by its circumsphere on the depth image. When a ray crosses a large cell only at its corner but not intersects with its circumsphere, this large cell is incorrectly decided as not intersecting with any rays without further splitting and query. Thus, the small cells at the corner space are falsely determined as unknown. With a higher map resolution, more small free cells in the corner space are determined as unknown, leading to a decreasing accuracy. To improve the accuracy of free space for such cases, it is recommended to use a smaller initial cell size $\mathtt{E}$ to avoid missing the cell corner. However, a trade-off must be considered as a smaller $\mathtt{E}$ might miss the opportunity to update on large cells, leading to slower occupancy updates. 

Notably, since we utilize Octomap as ground truth, the discretization error in Octomap also contributes to the accuracy loss in our benchmark experiment. Besides, despite the quantization error, Octomap has been widely used in real-world applications. Therefore, the accuracy achieved in our experiment, as reported in Table~\ref{tab:acc}, is adequate for various practical scenarios, as demonstrated in Section~\ref{sec:experiment}.

\subsubsection{Memory Consumption}
We evaluated the memory consumption of each mapping approach and reported the results in Table~\ref{tab:mem}. Given that SR\&CR(Grid) and SR\&CR(Octo) utilize the same map structure as GridMap and Octomap with equivalent mapping results, we only compare the memory consumption of D-Map against GridMap and Octomap. As shown in Table \ref{tab:mem}, Octomap exhibits the lowest memory consumption in most experiments, while GridMap consumes the highest. The memory consumption of D-Map follows that of Octomap closely. The excess memory consumption of D-Map mainly arises from using hashing grid map for maintaining occupied space. When high-resolution maps are built in large-scale environments (e.g., \SI{5}{cm} in \textit{Kitti} datasets), our D-Map exhibits better memory efficiency than Octomap due to its decremental property as well as the proper combination of grid-based and tree-based structures for different occupancy states.
\renewcommand*{\arraystretch}{1.0}
\begin{table*}[t]

	\centering
	\caption{Comparison of Memory Consumption (MB) at Different Resolution}
	\label{tab:mem}
	\setlength{\tabcolsep}{3.2pt}
	\begin{threeparttable}
            \begin{tabular}{@{}cccccccccccccccc@{}}
            \toprule
            \multicolumn{1}{l}{} & \multicolumn{5}{c}{\textit{Kitti\_04}}                                                  & \multicolumn{5}{c}{\textit{Kitti\_06}}                                                   & \multicolumn{5}{c}{\textit{Kitti\_07}}                                                  \\ \cmidrule(l){2-6}  \cmidrule(l){7-11} \cmidrule(l){12-16}
            Resolution (m)       & 1.0           & 0.5            & 0.25            & 0.1              & 0.05              & 1.0            & 0.5            & 0.25            & 0.1              & 0.05              & 1.0           & 0.5            & 0.25            & 0.1              & 0.05              \\ \midrule
            Grid Map             & 13.00         & 97.11          & 751.00          & 6964.45          & 33012.05          & 25.32          & 194.32         & 1532.03         & 14586.55         & 110589.78         & 22.49         & 105.35         & 894.36          & 7616.39          & 56745.76          \\
            Octomap              & \textbf{6.49} & \textbf{32.79} & \textbf{181.97} & \textbf{2212.85} & 14594.72          & \textbf{10.87} & \textbf{57.29} & \textbf{340.08} & \textbf{4587.48} & 31780.61          & \textbf{8.66} & \textbf{44.24} & \textbf{247.52} & \textbf{2858.18} & 18505.91          \\
            Ours                 & 18.19         & 75.05          & 282.98          & 2377.75          & \textbf{11990.52} & 26.36          & 135.79         & 620.75          & 4726.50          & \textbf{24082.73} & 20.80         & 89.68          & 402.93          & 3417.51          & \textbf{15315.92} \\
            \end{tabular}
	\end{threeparttable}
	\setlength{\tabcolsep}{8.1pt}
        \renewcommand*{\arraystretch}{1.0}        
	\begin{threeparttable}
            \begin{tabular}{@{}ccccccccccccc@{}}
            \toprule
            \multicolumn{1}{l}{} & \multicolumn{7}{c}{\textit{Workshop}}                                                                                  & \multicolumn{5}{c}{\textit{MainBuilding}}                                               \\ \cmidrule(l){2-8}  \cmidrule(l){9-13} 
            Resolution (m)       & 1.0           & 0.5           & 0.25          & 0.1            & 0.05            & 0.025           & 0.01              & 1.0           & 0.5            & 0.25            & 0.1              & 0.05              \\ \midrule
            Grid Map             & 0.18          & 1.32          & 6.50          & 102.35         & 826.43          & 6719.52         & 109724.66         & 22.31         & 106.70         & 817.38          & 7328.23          & 31258.35          \\
            Octomap              & \textbf{0.11} & \textbf{0.52} & \textbf{2.81} & \textbf{27.39} & \textbf{149.20} & \textbf{803.47} & \textbf{10619.32} & \textbf{5.83} & \textbf{33.28} & 261.84          & 3359.45          & \textbf{17566.79} \\
            Ours                 & 0.31          & 1.80          & 8.53          & 72.56          & 338.94          & 1516.09         & 11453.62          & 7.90          & 33.45          & \textbf{187.56} & \textbf{3038.41} & 26168.12         \\ 
            \end{tabular}
 \end{threeparttable}
	\setlength{\tabcolsep}{8.5pt}
	\begin{threeparttable}
            \begin{tabular}{@{}ccccccccccccc@{}}
            \toprule
            \multicolumn{1}{l}{} & \multicolumn{7}{c}{\textit{FR\_079}}                                                                                 & \multicolumn{5}{c}{\textit{Freibrug}}                                                   \\ \cmidrule(l){2-8}  \cmidrule(l){9-13} 
            Resolution (m)       & 1.0           & 0.5           & 0.25          & 0.1            & 0.05           & 0.025           & 0.01             & 1.0           & 0.5            & 0.25            & 0.1              & 0.05              \\ \midrule
            Grid Map             & 0.16          & 0.71          & 5.36          & 49.99          & 391.92         & 1921.08         & 26432.80         & 11.60         & 56.30          & 432.26          & 6594.09          & 30490.11          \\
            Octomap              & \textbf{0.08} & \textbf{0.34} & \textbf{1.64} & \textbf{14.32} & \textbf{99.17} & 797.95          & 9991.86          & \textbf{4.23} & \textbf{21.61} & \textbf{132.60} & 2012.67          & 14453.00          \\
            Ours                 & 0.25          & 1.02          & 4.38          & 26.89          & 120.96         & \textbf{600.16} & \textbf{3289.48} & 11.00         & 45.67          & 189.19          & \textbf{1721.62} & \textbf{14026.21} \\ \bottomrule
            \end{tabular}
	\end{threeparttable}
\vspace{-0.3cm}
\end{table*}

\section{Real-world Applications}
\label{sec:experiment}
We demonstrate the high efficiency of D-Map in two applications that require real-time high-resolution occupancy mapping on a high-resolution LiDAR. 

\subsection{Interactive Guidance for High-resolution Real-time 3D Mapping}\label{subsec:handheld}
With the recent emergence of 3D applications such as metaverse\cite{mystakidis2022metaverse,metaverse2022survey}, virtual and augmented reality\cite{cipresso2018VRAR}, and physical simulators\cite{shah2018airsim,kong2023marsim}, the demand for accurate and detailed 3D reconstructions of real-world environments has increased. The accuracy and completeness of such reconstructions depend heavily on the quality of the data collection process. To overcome this challenge, we have developed an interactive guidance system that leverages our D-Map to achieve high-resolution real-time 3D mapping. The system offers users information on the explored and unexplored areas, along with suggestions for the next mapping region. By utilizing this information, users can avoid rescanning the same areas repeatedly or skipping any unscanned areas, thereby improving their overall work efficiency. A video demonstrating the use of the system is available on YouTube: \href{https://youtu.be/m5QQPbkYYnA?t=251}{\tt youtu.be/m5QQPbkYYnA?t=251}.
\subsubsection{Experiment Setup}
\begin{figure}[t]
	\setlength\abovecaptionskip{-0.1\baselineskip}
	\centering
	\includegraphics[width=\linewidth]{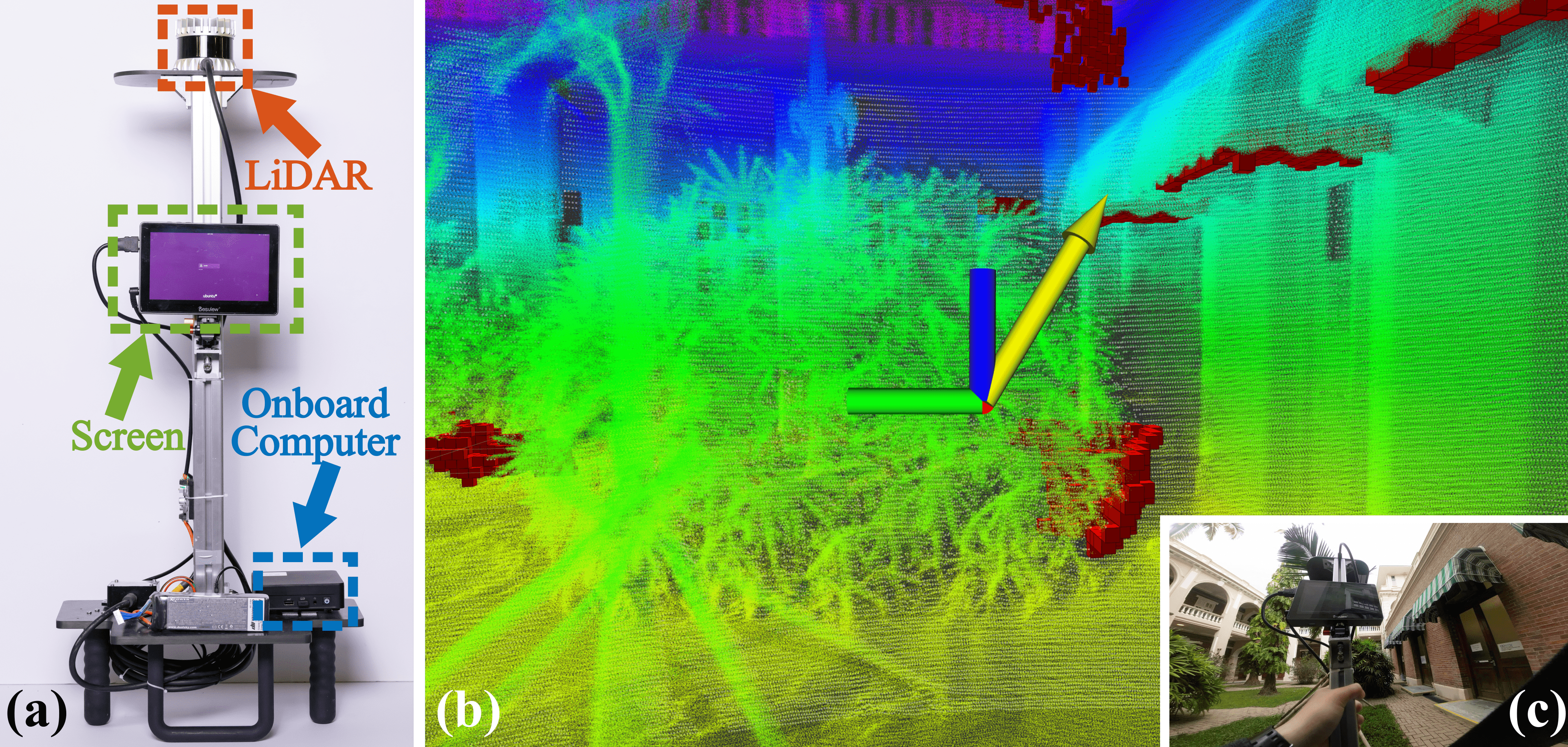}
	\caption{(a) Hardware setup of the handheld device for high-resolution 3D mapping, including an onboard computer (blue dashed block), a high-resolution LiDAR (orange dashed block), and a screen for online visualization (green dashed block). (b) A screenshot of the online visualization for interactively guiding the mapping process. The red cubes are the frontiers that users need to eliminate by scanning. The yellow arrow indicates the direction to the next suggested frontier for scanning. The depth measurements are accumulated and visualized by height value on the screen. (c) The first-person view of the user to conduct 3D mapping using our handheld device. }
	\label{fig:hardware} 
\end{figure} 

A handheld device is designed to conduct high-resolution 3D mapping, as shown in Fig.~\ref{fig:hardware}(a). The handheld device is equipped with an Intel NUC onboard computer with an i9-8550U CPU and \SI{64}{\giga\byte} memory, and an OS1-128 LiDAR integrated with an IMU. The OS1-128 LiDAR can output $2,621,440$ high-precision point clouds per second with a maximum detection range of \SI{120}{\meter} and a \SI{360}{\degree}$\times$\SI{45}{\degree} field of view. The frame rate of the OS1-128 LiDAR is set to \SI{10}{\hertz}.

We develop an interactive guidance system that consists of three modules: localization, mapping, and guidance. The localization module provides $6$ DoF sensor pose estimation using our previous work FAST-LIO2\cite{FASTLIO2}. Additionally, the point cloud acquired at each frame is compensated for motion distortion in FAST-LIO2 and transformed from the LiDAR frame to the world frame. The mapping module leverages our D-Map to update the occupancy map in real-time using the estimated LiDAR pose and the currently registered point cloud. Finally, the guidance module visualizes the frontiers, which correspond to the unknown space adjacent to the free space in the occupancy map, and suggests a direction to the user for complete scanning, as shown in Fig.~\ref{fig:hardware}(b).

The experiment aims to reconstruct an area measuring \SI{43}{\meter}$\times$\SI{19}{\meter}$\times$\SI{9}{\meter} located in Main Building, University of Hong Kong. We use a map resolution of \SI{10}{\cm} for occupancy mapping. The relax factor for D-Map is set to $\gamma = 1$ for full accuracy. In addition, we set the completeness threshold $\mathtt{\varepsilon}$ in Alg. \ref{alg:occupancy} to $0.8$, and the initial cell size $\mathtt{E}$ to \SI{1.6}{\meter}.

\subsubsection{Results}
\begin{figure}[t]
	\setlength\abovecaptionskip{-0.1\baselineskip}
	\centering
	\includegraphics[width=\linewidth]{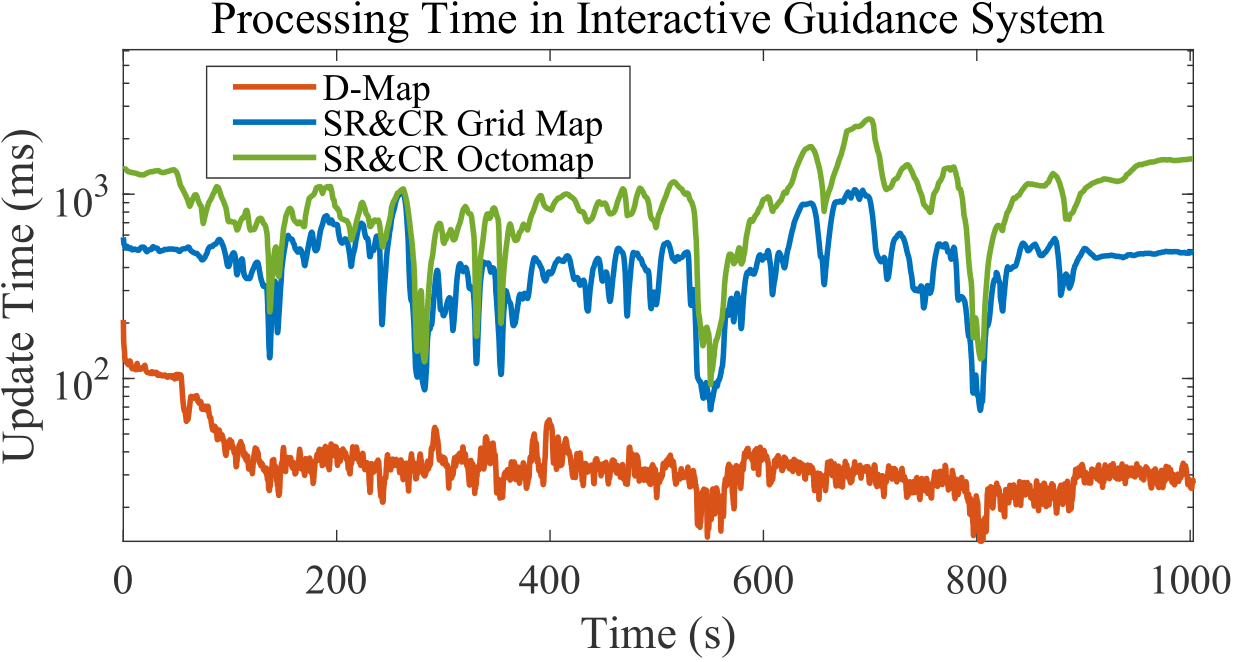}
	\caption{Comparison of the update time among our D-Map, SR\&CR(Grid), and SR\&CR(Octo). The update time of D-Map is acquired online in the interactive guidance system during the mapping process. The update time of the SR\&CR(Grid) and SR\&CR(Octo) is obtained by processing the recorded point clouds offline on the same computation platform.  }
	\label{fig:interactivetime} 
 \vspace{-0.4cm}
\end{figure} 

\begin{table}[t]
\centering
\caption{Occupancy Mapping Performance in Interactive Guidance System}
\label{tab:interact}
\setlength{\tabcolsep}{4.0pt}
    \begin{threeparttable}
    \begin{tabular}{@{}rccc@{}}
    \toprule
                    & Average Update Time & Processed Scan & Completion Rate\tnote{1} \\ \midrule
    D-Map & \SI{36.50}{\ms}     & 9980           & 99.46\%       \\
    SR\&CR(Grid) & \SI{348.58}{\ms}    & 2587           & 25.78\%       \\
    SR\&CR(Octo)  & \SI{721.79}{\ms}    & 1252           & 12.48\%       \\ \bottomrule
    \end{tabular}
    \begin{tablenotes}
    \footnotesize
    \item[1] The completion rate is calculated as the ratio of the processed scan over the total scan. In this experiment, LiDAR generated $10034$ scans of point clouds in total.
    \end{tablenotes}      
    \end{threeparttable}
   \vspace{-0.4cm}
\end{table}

As Fig.~\ref{fig:twoptcloudmap}(a) illustrates, the building reconstruction was completed with high fidelity. The mapping area was thoroughly scanned, except for an inaccessible classroom located in the middle of the map.

We conducted a performance evaluation of our D-Map against the super ray and culling region-based methods (i.e., SR\&CR(Grid) and SR\&CR(Octo)) on the recorded LiDAR data file that contains $10,034$ scans of LiDAR points in total. As shown in Fig.~\ref{fig:interactivetime} and Table~\ref{tab:interact}, D-Map achieved an average update time of only \SI{36.50}{\ms} to update the occupancy map, which is about $9.6$ times faster than SR\&CR(Grid) and $19.8$ times faster than SR\&CR(Octo). In addition, D-Map successfully processed the LiDAR data in real-time, except for the initial \SI{54}{\s} when the device was stationary. Consequently, D-Map processed \SI{99.46}{\percent} of the total scan, generating a high-resolution occupancy map with high fidelity. In contrast, SR\&CR(Grid) and SR\&CR(Octo) failed to process in real-time, processing only \SI{25.78}{\percent} and \SI{12.48}{\percent} of the total scan, respectively. As a result, they missed a significant amount of environment information required for interactive guidance.

\subsection{Autonomous UAV Exploration}
Unmanned aerial vehicles (UAVs) are becoming increasingly popular for autonomous exploration and scanning of real-world environments due to their unrestricted flight view and accessibility to hard-to-reach locations, such as caves \cite{uav_cave} and ancient remains \cite{uav_ancientremain}. However, the limited onboard computing power of UAVs presents a higher demand for efficient mapping modules compared to handheld devices. To address this demand, we have embedded our D-Map into a LiDAR-based UAV system, enabling it to autonomously explore complex scenes with higher-resolution mapping. This integration enables UAVs to achieve real-time mapping and guidance, thereby providing high-fidelity results even in challenging environments.

\subsubsection{Hardware System Setup}
 The UAV hardware platform includes an OS1-128 LiDAR, a Nora+ flight controller, and an onboard computer NUC 12 Pro Kit (i7-1260P CPU, maximum \SI{4.70}{\giga\hertz}, $12$-core, \SI{32}{\giga\byte} RAM), as shown in Fig. \ref{fig:coverfigure}(c). The extrinsic parameters between the LiDAR and the IMU on the flight controller are calibrated by LI-Init \cite{LIinit}. Additionally, we installed an action camera (DJI action 2) on the UAV to provide first-person view (FPV) images for better visualization. The small size of the UAV (\SI{40}{\cm}$\times$\SI{40}{\cm}$\times$\SI{21}{\cm}), combined with its large detection range and high scanning density, makes it well-suited for exploration and scanning tasks in complex scenes.

\begin{figure*}[t]

	\centering
	\setlength\abovecaptionskip{-0.1\baselineskip} 
	\includegraphics[width=\textwidth]{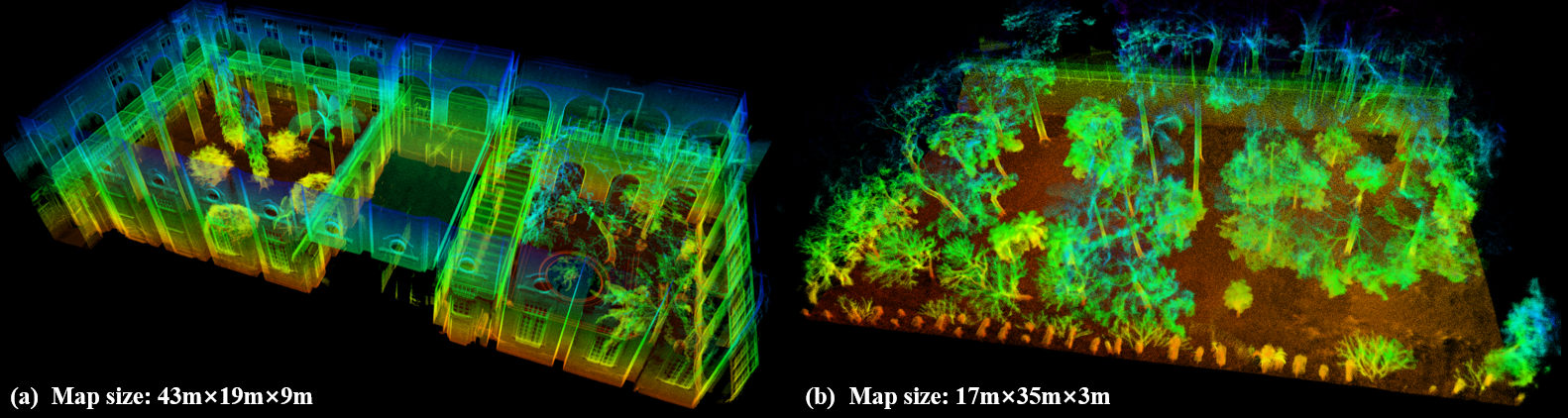}
	\caption{The high-fidelity point cloud map reconstructed from data collected by a high-resolution LiDAR. (a) Main Building in the University of Hong Kong, collected by a handheld device. (b) A forest in Hong Kong, autonomously collected by a UAV platform. }
	\label{fig:twoptcloudmap} 
 \vspace{-0.4cm}
\end{figure*}

\subsubsection{Software System Implementation} 
Multiple modules operate concurrently on the onboard computer to accomplish autonomous exploration tasks. The localization module employs FAST-LIO2\cite{FASTLIO2} to estimate the UAV's pose using data from the LiDAR and the IMU. This module generates undistorted point clouds, which are fed into the mapping module of the exploration planner. The mapping module utilizes D-Map for real-time occupancy mapping in place of the traditional ray-casting-based methods. The exploration planner, FUEL \cite{zhou2021fuel}, reads the information in the mapping module and calculates efficient, collision-free trajectories for exploration. These trajectories are tracked using an on-manifold model predictive controller\cite{mpc_luguozheng}. 

In this experiment, we used the same parameters for D-Map as in Section~\ref{subsec:handheld}. The maximum detection range of the LiDAR was set to \SI{15}{\meter} to ease the computation load in the exploration planner.

\subsubsection{Results}
The real-flight autonomous exploration and scanning experiments were conducted in two complex scenarios: an abandoned fortress site built about 100 years ago (\SI{20.5}{\meter}$\times$\SI{16}{\meter}$\times$\SI{6}{\meter}) and a natural forest (\SI{17}{\meter}$\times$\SI{35}{\meter}$\times$\SI{3}{\meter}). The entire exploration process is available on \href{https://youtu.be/m5QQPbkYYnA?t=78}{\tt youtu.be/m5QQPbkYYnA?t=78}. We successfully carried out three flight experiments in each scenario. 
The UAV completed high-precision scanning in these complex environments, and the average flight time for autonomous exploration and scanning by the UAV was \SI{123}{\second} and \SI{163}{\second}, respectively. The acquired point clouds are shown in Fig. \ref{fig:coverfigure}(a) and Fig. \ref{fig:twoptcloudmap}(b), exhibiting high precision and completeness. 

We compared the running times of the mapping module employing D-Map and a uniform grid map utilizing ray-casting, which was the original method used in FUEL \cite{zhou2021fuel}, on the same computing platform. As depicted in Fig. \ref{fig:realflightrunningtime}, D-Map substantially reduced  the time required for the mapping module, thereby demonstrating its effectiveness and efficiency for real-time occupancy mapping on massive streaming point clouds of over two million per second.

\begin{figure}[t]
	\setlength\abovecaptionskip{-0.1\baselineskip}
	\centering
	\includegraphics[width=0.9\linewidth]{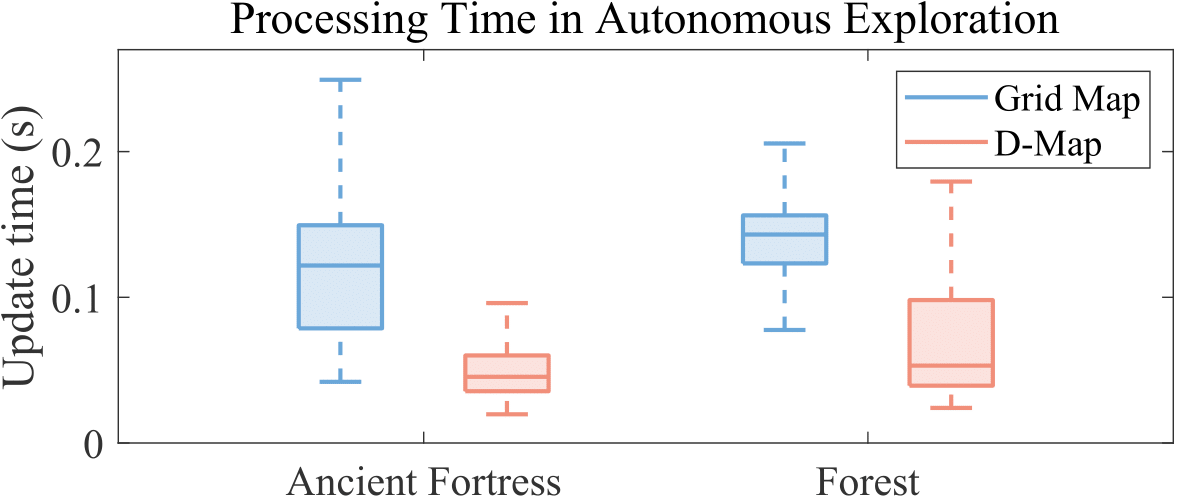}
	\caption{The comparison of processing time between a ray-casting-based grid map and our D-Map.}
	\label{fig:realflightrunningtime} 
 \vspace{-0.3cm}
\end{figure} 
\section{Extensions}
\subsection{Occupancy Mapping in Large-scale Environment}
\begin{table}[t]
	\setlength\abovecaptionskip{-0.1\baselineskip}
	\centering
	\caption{Comparison of Update Time (ms) at Different Resolutions on a Large-scale Dataset}
	\label{tab:ford_eff}
	\setlength{\tabcolsep}{2.6pt}
 \renewcommand*{\arraystretch}{1.0}
	\begin{threeparttable}
\begin{tabular}{@{}ccccccccc@{}}
\toprule
\multicolumn{1}{l}{} & \multicolumn{4}{c}{\textit{ford\_1}}                            & \multicolumn{4}{c}{\textit{ford\_2}}                            \\ \cmidrule(l){2-5}  \cmidrule(l){6-9}
Resolution [m]       & 1.0           & 0.5           & 0.25           & 0.15           & 1.0           & 0.5           & 0.25           & 0.15           \\ \midrule
GridMap             & 21.0          & 42.7          & 153.4          & $\,\,\,\times^1$              & 18.3          & 40.5          & 133.9          & $\times$              \\
Octomap              & 186.8         & 305.5         & 570.9          & 1312.6         & 176.5         & 289.3         & 533.1          & 966.0          \\
SR\&CR(Grid)      & 24.7          & 64.1          & 233.2          & $\times$              & 22.9          & 58.9          & 208.2          & $\times$              \\
SR\&CR(Octo)       & 34.8          & 105.4         & 373.1          & 1197.9         & 30.2          & 103.6         & 391.8          & $\times$              \\
Ours                 & \textbf{19.1} & \textbf{41.1} & \textbf{121.0} & \textbf{277.9} & \textbf{17.2} & \textbf{36.5} & \textbf{111.3} & \textbf{278.3} \\ \bottomrule
\end{tabular}
		\begin{tablenotes}
				\footnotesize
				\item[1]
				``$\times$'' denotes that this method failed due to exceeding memory limitation (i.e., \SI{64}{\giga\byte} RAM + \SI{64}{\giga\byte} swap memory).
        \end{tablenotes}
	\end{threeparttable}	
\vspace{-0.3cm}
\end{table}
To evaluate the performance of our D-Map in large-scale occupancy mapping, we conduct an experiment using two sequences from Ford Multi-AV Seasonal Dataset\cite{agarwal2020ford}. These two sequences, originally namely ``2017-10-26-V2-Log1'' and ``2017-10-26-V2-Log2'', are referred to as \textit{ford\_1} and \textit{ford\_2}, respectively. The mapping areas for these sequences measure \SI{8090}{\meter}$\times$\SI{11494}{\meter}$\times$\SI{96}{\meter} and \SI{8107}{\meter}$\times$\SI{11659}{\meter}$\times$\SI{103}{\meter}. We compare our D-Map with four benchmark methods, as discussed in Section~\ref{sec:benchmark}. The corresponding results are presented in Table~\ref{tab:ford_eff}. The findings indicate that D-Map consistently outperforms the other methods across all resolutions in the context of large-scale applications. Notably, D-Map exhibits superior memory efficiency, enabling it to handle a resolution of \SI{0.15}{\meter} within such expansive mapping areas. In contrast, grid-based maps (i.e., GridMap and SR\&CR(Grid)) fail to meet the memory limitations imposed by the device. 
\subsection{Map Region Sliding for D-Map}
\begin{figure}[t]
	\setlength\abovecaptionskip{-0.1\baselineskip}
	\centering
	\includegraphics[width=\linewidth]{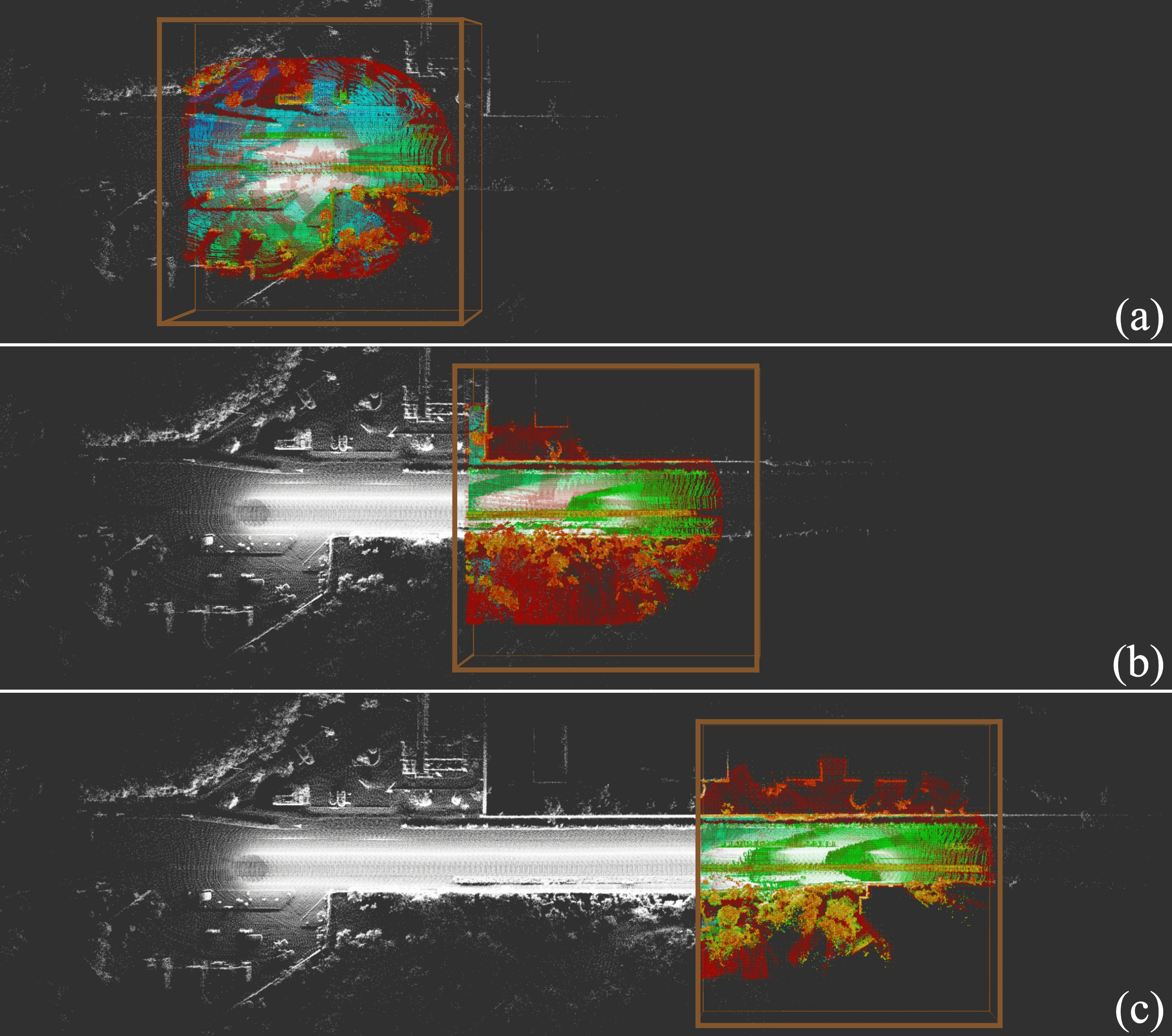}
	\caption{A demonstration of D-Map with map region sliding. Figures (a), (b), and (c) showcase the mapping process as the vehicle moves from left to right, with the mapping region sliding accordingly. The white point clouds represent the accumulated historical point clouds. The colored space within the mapping region of D-Map represents the occupancy information. An axis-aligned bounding box is employed to provide a visual representation of the mapping region, outlined by orange lines. }
	\label{fig:mapsliding} 
 \vspace{-0.2cm}
\end{figure}
In the context of occupancy mapping in even larger-scale environments, the issue of memory consumption becomes a significant concern. Therefore, we introduce a map region sliding technique in D-Map to address this challenge. This approach allows the removal of distant occupancy information from the current vehicle position. Fig.~\ref{fig:mapsliding} illustrates the concept of map region sliding, wherein the mapping region of D-Map slides alongside the vehicle. The sliding process adjusts the mapping region when the vehicle has moved beyond a distance threshold and removes information outside of the mapping region from the D-Map. By dynamically updating the mapping region, D-Map optimizes memory usage and only retains the relevant occupancy information. This mechanism enables our method to handle large-scale environments while mitigating the impact of memory constraints. 

\subsection{Extension to Range Sensors with Measurement Noise}
In D-Map, the on-tree update strategy leverages the high precision of depth measurements from LiDAR sensors to remove known cells on the octree directly. We further extend D-Map to range sensors with higher measurement noise (e.g., depth cameras) by incorporating occupancy probabilities. At each update, the occupancy probabilities on the correspondent nodes decrease for cells determined as known, while the occupancy probabilities in the corresponding cells of point clouds in the hashing grid map increase. The removal of D-Map is disabled to obtain correct occupancy probabilities. When querying the map, D-Map summarizes the occupancy probabilities from the hybrid map structure to determine occupancy states. This adaptation of occupancy probabilities appropriately handles measurement noise similar to other existing occupancy mapping approaches while retaining the superior efficiency of D-Map. We conduct several experiments to validate the performance of D-Map when incorporating occupancy probabilities.

\subsubsection{Qualitative Evaluation}
We conducted a qualitative comparison of the mapping results obtained from the original D-Map and D-Map with occupancy probability against those obtained from Octomap. To evaluate our approach, we used the \textit{pioneer\_slam3} sequence in the TUM dataset\cite{tum2012Sturm}, which was captured by a Kinect depth camera mounted on a Pioneer robot. The original D-Map used the same parameters as those described in Section~\ref{sec:benchmark}. For D-Map with occupancy probability and Octomap, the hit and miss probabilities by a ray for occupied and free space are set to $0.7$ and $0.4$, respectively, while the occupancy threshold for determining a cell as occupied is set to $0.9$. The mapping results are presented in Fig.~\ref{fig:TUM_map}. The results of the original D-Map and D-Map with occupancy probability (Fig.~\ref{fig:TUM_map}(a) and (b), respectively) highlight the effective handling of measurement noise from the depth camera. Furthermore, the mapping result obtained by D-Map with occupancy probability (Fig.~\ref{fig:TUM_map}(b)) exhibits few discrepancies from that produced by Octomap (Fig.~\ref{fig:TUM_map}(c)), indicating our accurate mapping performance.  
\subsubsection{Efficiency}
We conduct benchmark experiments to evaluate our efficiency on depth cameras. We compare the update time of our original D-Map and D-Map with occupancy probability against the four mapping methods in Section~\ref{sec:benchmark}, using four sequences in the TUM dataset \cite{tum2012Sturm}: \textit{pioneer\_360}, \textit{pioneer\_slam}, \textit{pioneer\_slam2}, and \textit{pioneer\_slam3}. The results are presented in Table~\ref{tab:tum_eff}, where the original D-Map and D-Map with occupancy probability are referred to as ``Ours'' and ``Ours(prob)'' respectively. Super Ray and Culling Region-based Grid Map (i.e., SR\&CR(Grid)) performs the best at resolutions of \SI{1.0}{\meter} and \SI{0.5}{\meter}. However, at high resolutions (i.e., $d<0.5\mathrm{m}$), our original D-Map and D-Map with occupancy probability have superior performance. On average, the original D-Map and D-Map with occupancy probability achieve approximately $9.1$ and $5.3$ times faster than the fastest competing method (i.e., GridMap) at the resolution of \SI{5}{\cm}, respectively. The update time of D-Map with occupancy probability is comparable to that of the original D-Map at low resolutions (i.e., $d\geqslant0.25\mathrm{m}$). However, at high resolutions (i.e., $d\leqslant0.1\mathrm{m}$), D-Map with occupancy probability is slower than the original D-Map since a large number of cells are kept on the octree to maintain occupancy probabilities.
\subsubsection{Accuracy}
In addition to the previous experiments, we conduct simulation experiments to evaluate the accuracy of D-Map with occupancy probability. The simulation environment has a dimension of \SI{30}{\meter}$\times$\SI{20}{\meter}$\times$\SI{4}{\meter}, as depicted in Fig.~\ref{fig:auroc_sim}(a). We evaluate the Area Under the Receiver Operating Characteristic (AUROC) curves of D-Map with occupancy probability, Grid Map, and Octomap by removing different fractions of observed data, with results provided in Fig.~\ref{fig:auroc_sim}(b). The results show that our D-Map with occupancy probability has a similar mapping performance as Octomap and Grid Map.

\begin{figure*}[t]
	\setlength\abovecaptionskip{-0.1\baselineskip}
	\centering
	\includegraphics[width=\textwidth]{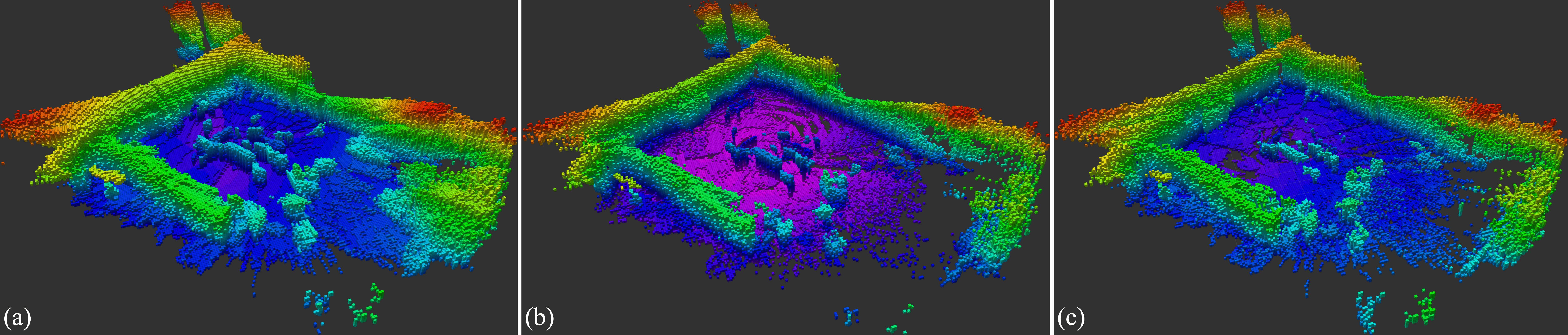}
	\caption{The mapping results of the sequence \textit{pioneer\_slam3} in TUM dataset. (a) The original D-Map (b) D-Map with occupancy probability. (c) Octomap}
	\label{fig:TUM_map} 
 \vspace{-0.3cm} 
\end{figure*} 
\begin{table*}[t]
	\setlength\abovecaptionskip{-0.1\baselineskip}
	\centering
	\caption{Comparison of Update Time (ms) at Different Resolution on TUM dataset}
	\label{tab:tum_eff}
	\setlength{\tabcolsep}{2.5pt}
 \renewcommand*{\arraystretch}{0.9}
	\begin{threeparttable}
\begin{tabular}{@{}ccccccccccccccccccccc@{}}
\toprule
               & \multicolumn{5}{c}{\textit{pioneer\_360}}                                    & \multicolumn{5}{c}{\textit{pioneer\_slam}}                                   & \multicolumn{5}{c}{\textit{pioneer\_slam2}}                                  & \multicolumn{5}{c}{\textit{pioneer\_slam3}}                                 \\ \cmidrule(l){2-6}  \cmidrule(l){7-11} \cmidrule(l){12-16}  \cmidrule(l){17-21}
Resolution [m] & 1.0          & 0.5           & 0.25          & 0.1           & 0.05          & 1.0          & 0.5           & 0.25          & 0.1           & 0.05          & 1.0          & 0.5           & 0.25          & 0.1           & 0.05          & 1.0          & 0.5          & 0.25          & 0.1           & 0.05          \\ \midrule
GridMap          & 29.7         & 48.3          & 85.1          & 195.8         & 383.2         & 30.4         & 47.1          & 83.1          & 186.7         & 377.9         & 32.7         & 51.8          & 90.9          & 211.5         & 408.0         & 28.5         & 43.2         & 76.4          & 173.0         & 334.3         \\
Octomap        & 93.3         & 154.0         & 273.3         & 611.9         & 1138.7        & 91.3         & 148.9         & 262.8         & 582.1         & 1135.2        & 106.5        & 179.4         & 318.4         & 699.7         & 1237.6        & 83.7         & 138.5        & 246.5         & 542.2         & 1001.4        \\
SR\&CR(Grid)    & 6.8          & \textbf{10.2} & 34.7          & 180.5         & 482.9         & \textbf{7.0} & 11.0          & 35.2          & 169.8         & 472.7         & \textbf{7.5} & 11.8          & 38.3          & 202.7         & 532.5         & 6.4          & 9.5          & 31.2          & 159.4         & 436.6         \\
SR\&CR(Octo)     & \textbf{6.7} & 10.3          & 31.1          & 192.7         & 485.9         & 7.1          & \textbf{10.5} & 31.3          & 184.7         & 514.7         & 7.6          & \textbf{11.3} & 35.2          & 213.1         & 508.0         & \textbf{6.2} & \textbf{9.2} & 28.0          & 172.4         & 464.6         \\
Ours(prob)    & 26.3         & 25.7          & \textbf{27.5} & 37.6          & 73.9          & 29.0         & 27.7          & \textbf{29.5} & 38.5          & 70.3          & 28.6         & 29.2          & \textbf{30.2} & 39.6          & 70.8          & 25.3         & 26.1         & \textbf{27.3} & 36.0          & 66.5          \\
Ours           & 27.0         & 26.8          & 27.8          & \textbf{31.5} & \textbf{41.2} & 30.5         & 28.9          & 29.9          & \textbf{33.2} & \textbf{40.9} & 30.3         & 30.4          & 30.7          & \textbf{34.5} & \textbf{43.2} & 26.8         & 27.3         & 27.7          & \textbf{31.0} & \textbf{38.9} \\ \bottomrule
\end{tabular}
	\end{threeparttable}	
 \vspace{-0.3cm}
\end{table*}
\begin{figure}[t]
	\setlength\abovecaptionskip{-0.1\baselineskip}
	\centering
	\includegraphics[width=\linewidth]{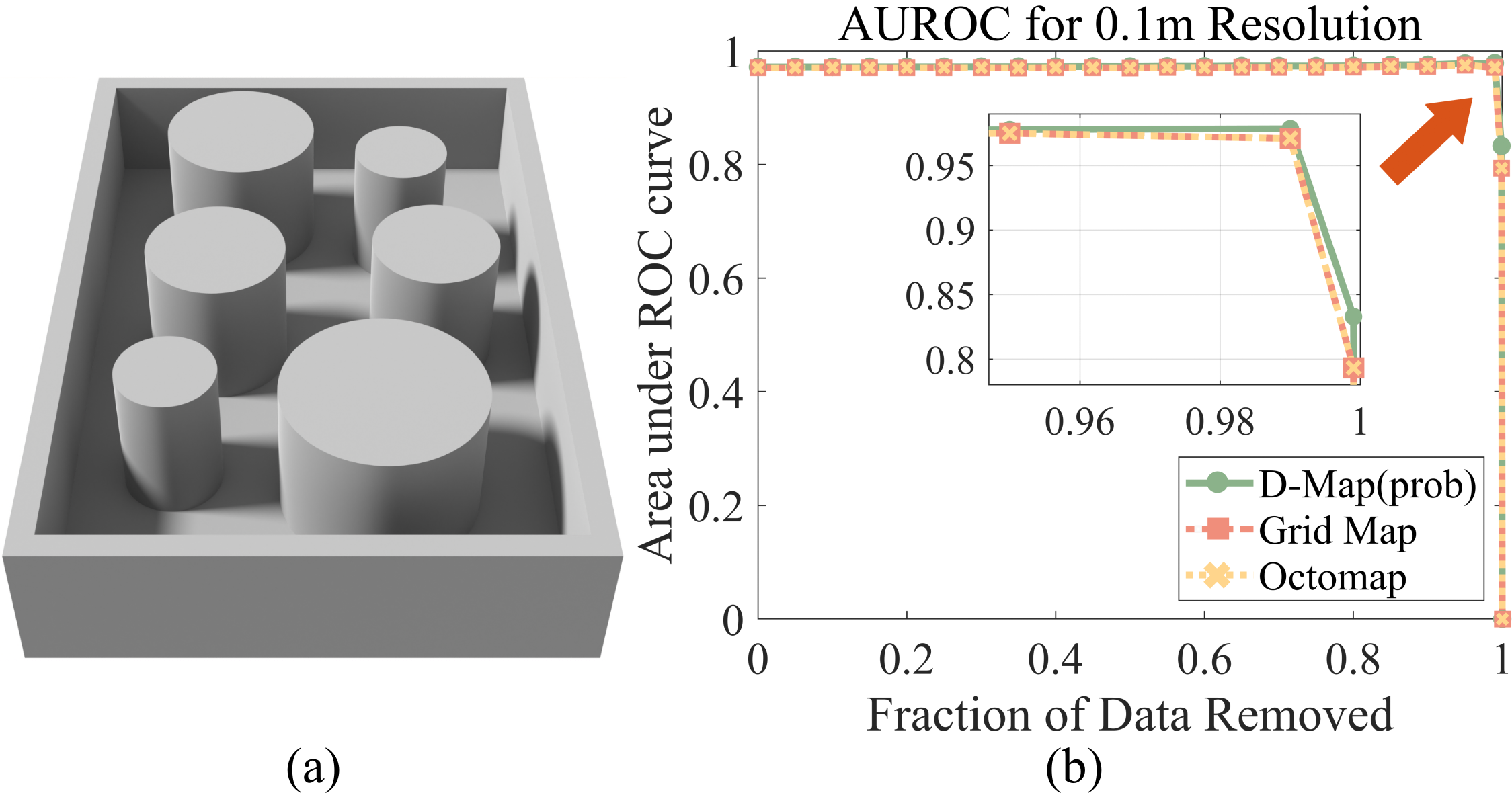}
	\caption{(a) The simulation environment for accuracy evaluation. (b) The AUROC curves of D-Map with occupancy probability, Grid Map, and Octomap for different fractions of removed data. }
	\label{fig:auroc_sim} 
 \vspace{-0.4cm}
\end{figure} 
\section{Discussion}\label{sec:discussion}
In this section, we first discuss occupancy mapping using a depth image in terms of efficiency and accuracy, followed by the discussion on parallel processing over D-Map.
\subsection{Occupancy Mapping on Depth Image}
\subsubsection{Efficiency}
ray-casting is an indispensable component in the existing occupancy mapping framework for occupancy updates. However, the computational demands of ray-casting increase with the number of point clouds and the longer detection range of high-resolution LiDARs, which makes it unsuitable for computationally limited robotic applications. To alleviate the increasing computation load in occupancy mapping, we propose an alternative approach that leverages a depth image created from point clouds to determine occupancy states. Moreover, we design a novel on-tree occupancy update strategy that exploits the hierarchical structure of the octree to determine the occupancy state of larger cells, avoiding the need for timely updates on smaller cells as required by ray-casting-based methods. Additionally, we remove known cells from the octree, reducing the map size and further lowering the update cost. The significant reduction in both the number of cells to be updated and the cost of map update renders substantially high efficiency in D-Map, as verified in the benchmark experiments in Section~\ref{sec:benchmark}. 

\subsubsection{Accuracy}
The accuracy of D-Map can be affected by two primary factors: depth map rasterization and occupancy state determination, which are both related to the depth map resolution. 

In the rasterization process, using a low-resolution depth image might result in the loss of depth measurements since a pixel only keeps one point with the smallest depth, as described in Section~\ref{subsec:proj}. In the occupancy state determination process, D-Map determines occupancy states in spherical coordinates on the depth image rather than Cartesian coordinates. Several design choices affect the accuracy of D-Map, including 1) projecting cells onto a depth image using their circumsphere radius, which can result in information loss at the corners of the cells; 2) discretizing the projected area into pixels on a depth image, which can result in a distorted shape in 3D space; and 3) using an occupancy state determination method that allows for early determination of large regions if the observation completeness $\alpha$ surpasses a threshold $\varepsilon$, which may cause errors in the unobserved regions. 

Despite a slight decrease in accuracy, D-Map provides comparable accuracy with the existing mapping approaches while achieving higher efficiency thanks to the comprehensive analysis and appropriate selection of depth image resolution.   

\subsection{Parallel Processing over D-Map}
Although the primary focus during the design of D-Map does not prioritize parallel processing, certain aspects of the update process in D-Map lend themselves well to parallel implementation. Specifically, 1) The projection of a point cloud onto a depth image can be parallelized at the pixel level. 2) The update process of the hashing grid map naturally lends itself to parallel processing, allowing concurrent updates to different grid cells. However, it is worth noting that the octree structure employed in D-Map poses challenges for parallelization. Parallelizing the octree's operations is not straightforward due to its hierarchical nature. As an alternative, we suggest utilizing a B-tree\cite{gpubtree} to exploit parallel processing on efficient management of 3D data within D-Map.
\section{Conclusion}\label{sec:conclusion}
This paper proposes a novel framework for occupancy mapping termed D-Map, which aims to provide efficient occupancy updates for high-resolution LiDAR sensors. Our proposed framework consists of three key techniques. Firstly, a method has been proposed to determine the occupancy state of a cell at arbitrary size through depth image projection. Secondly, a hybrid map structure has been developed with an efficient on-tree update strategy. Thirdly, a removal strategy has been introduced, which utilizes the low false alarm rate of LiDARs to remove known cells from the map. These techniques work in conjunction to reduce the number of cells that need updating and lower the cost of map updates, resulting in a significant improvement in efficiency.

To validate our proposed framework, we provide theoretical analyses of its accuracy and efficiency and conduct extensive benchmark experiments on various LiDAR datasets. The results show that D-Map substantially improves efficiency against other state-of-the-art mapping methods while maintaining comparable accuracy and high memory efficiency. Two real-world applications were demonstrated to showcase the effectiveness and efficiency of D-Map for high-resolution LiDAR-based applications. 

In the future, we could extend our D-Map framework to support occupancy mapping in dynamic environments and exploit parallel processing. 

\section*{Acknowledgement}
The authors would like to thank DJI Innovation Technology for the equipment support during the whole work. They would like to thank Ms. Yuhan Xie for the helpful discussions and Prof. Ximin Lyu for the support of the experiment site. They would also like to thank Mr. Wendi Dong for his help with the handheld equipment. 

\bibliography{main}



\section*{Supplementary material}

# Supplementary Materials for Occupancy Grid Mapping without Ray-Casting for High-resolution LiDAR sensors

Yixi Cai, Fanze Kong, Yunfan Ren, Fangcheng Zhu, Jiarong Lin, Fu Zhang

## I. PROOF OF THEOREM 1

*Proof.* The volume of free space $V$ is composed of two parts: center free space $V_{\text{center}}$ and free space of rays $V_{\text{ray}}$, as shown in Fig. 1. The radius $r$ of the center sphere is obtained in consideration of vertical FoV $\Phi \in [-\alpha, \alpha]$ as follows:

$$r = \sqrt{3}\sin(\min\{\operatorname{atan}(\frac{1}{\sqrt{2}}) + \alpha, \frac{\pi}{2}\})R/\gamma \qquad (1)$$

An illustration to calculate the radius $r$ is presented in Fig. 2. Let $\gamma_0 = \sqrt{3}\sin(\min\{\operatorname{atan}(\frac{1}{\sqrt{2}}) + \alpha, \frac{\pi}{2}\})$, then $r = \frac{\gamma_0 R}{\gamma}$. The volume of center space $V_{\text{center}}$ can be obtained in proportion to the entire free space $V_I$ as

$$V_{\text{center}} = (\frac{r}{R})^3 V_I = (\frac{\gamma_0}{\gamma})^3 V_I \qquad (2)$$

Since the volume of center space $V_{\text{center}}$ is smaller than the entire free space, $\gamma_0/\gamma$ should be smaller than 1, leading to $\gamma > \gamma_0$. Otherwise, we have $f(\gamma) = 1.0$, $0 < \gamma \leqslant \gamma_0$.

For the volume of free space of rays $V_{\text{ray}}$, we consider the points $\mathcal{P} = \{p_{i,j} | i = 1, \ldots, N, j = 1, \ldots, \mathcal{M}_{\text{ray}}\}$ on the depth image $\mathcal{M}$ where $N$ and $M$ are the sizes of two dimension of the depth image. For each point $p_{i,j}$, the number of cells traversed by a ray from the origin to the point is the same as counting the cells intersected with the diagonal of a cuboid $\mathcal{C}$ as:

$$\begin{aligned}\mathcal{N}_{\text{cell}} =& a + b + c - \gcd(a,b) \\ & - \gcd(b,c) - \gcd(a,c) + \gcd(a,b,c)\end{aligned} \qquad (3)$$

where $a, b, c$ are the three dimensions of the cuboid. Suppose the point $p_{i,j}$ is projected to the depth image $\mathcal{M}$ with the spherical angle of $(\phi_i, \theta_j)$, the size of the equivalent cuboid $\mathcal{C}$ is

$$\begin{aligned}a = \frac{R-r}{d}\cos(\phi_i)\cos(\theta_i),\ b = \frac{R-r}{d}\cos(\phi_i)\sin(\theta_i) \\ c = \frac{R-r}{d}\sin(\phi_i)\end{aligned} \qquad (4)$$

Thus, the number of cells traversed by a ray from the origin to a point $p_{i,j}$ is obtained by approximation of (3) as:

$$\mathcal{N}_{\text{cell},i,j} = \frac{(R-r)}{d}\bigl(\text{c}(\phi_i)(\text{c}(\theta_j)+\text{s}(\theta_j))+\text{s}(\phi_i)\bigr) \qquad (5)$$

Yixi Cai and Fanze Kong contributed equally to this work.
Corresponding author: Fu Zhang.
The authors are with Mechatronics and Robotic Systems (MaRS) Laboratory, Department of Mechanical Engineering, University of Hong Kong, Hong Kong SAR, China. (email: {yixicai, kongfz, renyf, zhufc, jiarong.lin}@connect.hku.hk; fuzhang@hku.hk)

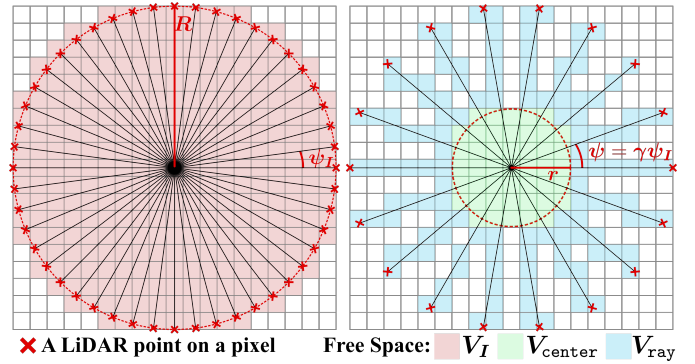

Fig. 1. (a) Free space $V_I$ determined by projecting to $\mathcal{M}_I$ with a resolution of $\psi_I$. (b) Free space $V$ determined by projecting to $\mathcal{M}$ with a resolution of $\psi = \gamma\psi_I$, which is composed of $V_{\text{center}}$ and $V_{\text{ray}}$.

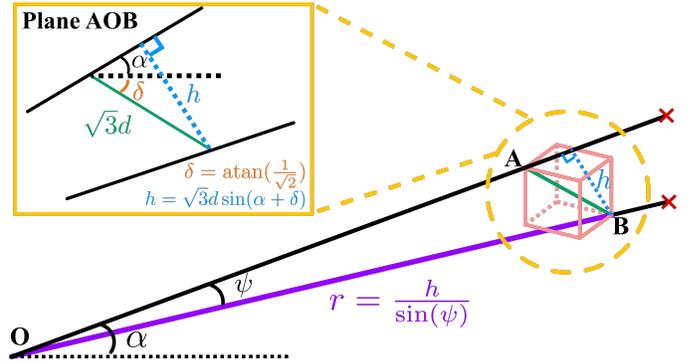

Fig. 2. When casting rays from the sensor origin $\mathbf{O}$ to the LiDAR points on two neighbor pixels, there exist points $\mathbf{A}$ and $\mathbf{B}$ on the rays that the length of $\mathbf{AB}$ is large enough to contain a cell of size $d$. In this case, the length $\mathbf{AB}$ equals the diagonal of the cell, which is $\sqrt{3}d$, and the calculation of $\mathbf{OB}$ is conducted on plane $\mathbf{AOB}$. The length of $\mathbf{OB}$ is the center radius $r$.

where $\text{c}(\cdot)$ and $\text{s}(\cdot)$ denote $\cos(\cdot)$ and $\sin(\cdot)$, respectively. As the volume of each cell is $d^3$, the total volume of these cells is derived as

$$\boldsymbol{V}_{\text{ray},i,j} = (1 - \frac{\gamma_0}{\gamma})\bigl(\text{c}(\phi_i)(\text{c}(\theta_j)+\text{s}(\theta_j))+\text{s}(\phi_i)\bigr)Rd^2 \qquad (6)$$

When summing the volume of free space traversed by each ray, we consider $1/8$ of the points $\mathcal{P}$ for simplicity, which locates in an octant of free space $V$ corresponding to the FoV angle of $[0, \pi/2] \times [0, \alpha]$. The depth image dimension corresponding to the octant space is calculated as:

$$N \approx \frac{\alpha}{\gamma\psi_I}, \quad M \approx \frac{\pi}{2\gamma\psi_I} \qquad (7)$$

The spherical coordinate $(\phi_i, \theta_j)$ of $p_{i,j}$ can be obtained from

the depth image resolution $\psi$ as follows:

$$\phi_i = i\gamma\psi_I, \ \theta_j = j\gamma\psi_I \quad (8)$$

Summarizing the volume of free space covered by the rays from the origin to the point cloud $\mathcal{P}$ yields:

$$\boldsymbol{V}_{\text{ray}} = 8\sum_{i=0}^{N-1}\sum_{j=0}^{M-1} \boldsymbol{V}_{\text{ray},i,j} = (4\pi(1-\text{c}(\alpha))+16\text{s}(\alpha))(1-\frac{\gamma_0}{\gamma})\frac{R^3}{\gamma^2} \quad (9)$$

The volume of free space $\boldsymbol{V}_I$ is equivalent to the volume of LiDAR FoV, which can be obtained by the volume of a sphere cut by two sphere caps and two cones, which is

$$\boldsymbol{V}_I = \frac{4\pi R^3}{3} - \frac{2\pi R^3}{3}(\text{s}(\alpha)\text{c}^2(\alpha) - (1-\text{s}(\alpha))^2(2+\text{s}(\alpha))) \quad (10)$$

Finally, the accuracy function for the 3-Dimensional case is

$$f(\gamma) = \frac{\boldsymbol{V}_{\text{center}}+\boldsymbol{V}_{\text{ray}}}{\boldsymbol{V}_I} = (\frac{\gamma_0}{\gamma})^3 + A(1-\frac{\gamma_0}{\gamma})\frac{1}{\gamma^2}, \ \gamma > \gamma_0 \quad (11)$$

where $A$ is a constant factor relative to LiDAR's vertical FoV range $\alpha$

$$A = \frac{3(1-\text{c}(\alpha)) + \frac{12}{\pi}\text{s}(\alpha)}{1 - \frac{1}{2}(\text{s}(\alpha)\text{c}^2(\alpha) + (1-\text{s}(\alpha))^2(2+\text{s}(\alpha)))} \quad (12)$$

Summarizing the results above, the accuracy function $f(\gamma)$ is provided as

$$f(\gamma) = \begin{cases} \frac{\gamma_0^3}{\gamma} + A(1-\frac{\gamma_0}{\gamma})\frac{1}{\gamma^2} & \gamma > \gamma_0 \\ 1 & 0 < \gamma \leqslant \gamma_0 \end{cases} \quad (13)$$

∎

## II. PROOFS FOR TIME COMPLEXITY ANALYSIS

### A. Proof of Lemma 1

*Proof.* The on-tree update strategy searches the nodes inside or intersecting with the sensing area. To analyze the time complexity, we examine the worst-case scenario for updating the D-Map, which occurs when the octree splits the sensing area into the smallest size. This situation can arise when the point cloud has a serrated shape (as shown in Fig.9(a) in the primary article), leading to a depth image resembling a checkerboard where there are abrupt changes in depth values from a minimum to a maximum between neighboring pixels. In such situations, the number of leaf nodes visited in the D-Map is equivalent to the number of cells traversed by rays of points on the depth image. For each ray, the number of traversed cells is approximated as $\mathcal{O}(R/d)$, as derived in (3) and (5). Thus, the total number of leaf nodes in the worst case is given as:

$$\mathcal{S}_{\text{leaf}} = \mathcal{O}(\frac{nR}{d}) \quad (14)$$

Next, we consider the other nodes visited along the path toward the leaf nodes. The number of these accompanying nodes $\mathcal{S}_{\text{acc}}$ is determined by the height of the octree, which is $\log(D/d)$, as follows:

$$\mathcal{S}_{\text{acc}} = \mathcal{O}(\frac{nR}{d}\log(\frac{D}{d})) \quad (15)$$

Therefore, the total number of visited nodes in the worst case of our on-tree update strategy is obtained as

$$\begin{aligned} \mathcal{S}_{\text{DMap}} &= \mathcal{S}_{\text{leaf}} + \mathcal{S}_{\text{acc}} \\ &= \mathcal{O}(\frac{nR}{d}) + \mathcal{O}(\frac{nR}{d}\log(\frac{D}{d})) \\ &= \mathcal{O}(\frac{nR}{d}\log(\frac{D}{d})) \end{aligned} \quad (16)$$

∎

### B. Proof of Lemma 2

*Proof.* When considering the procedure of occupancy state determination, the time complexity is dominated by the range query on the 2-D segment tree. As explained in Section IV-B in the primary article, the time complexity of a range query on a 2-D segment tree is $\mathcal{O}(\log N \log M)$ where $N$ and $M$ are the sizes of two dimensions. In our case, the 2-D segment tree is constructed from a depth image, with $N = \lceil \Phi_I/\psi_I \rceil$ and $M = \lceil \Theta_I/\psi_I \rceil$. Thus, the time complexity of occupancy state determination is $\mathcal{O}(\log(\Phi_I/\psi_I)\log(\Theta_I/\psi_I))$. ∎

### C. Proof of Theorem 2

*Proof.* In the on-tree update strategy, an occupancy state determination process is applied to each node that is visited. The time complexity of this process is given by Lemma 2. Additionally, the number of visited nodes is given by $\mathcal{S}_{\text{DMap}}$ in Lemma 1. Thus, the overall time complexity can be expressed as the product of these two factors:

$$\mathcal{O}\Big(\frac{nR}{d}\log(\frac{D}{d})\log(\frac{\Phi_I}{\psi_I})\log(\frac{\Theta_I}{\psi_I})\Big) \quad (17)$$

In D-Map, the depth image resolution is bounded by (2) as:

$$\psi_I = \max\{\frac{d}{R}, \psi_{\text{lidar}}\} \geqslant \psi_{\text{lidar}} \quad (18)$$

Therefore, the time complexity in (17) can be written as:

$$\mathcal{O}\Big(\frac{nR}{d}\log(\frac{D}{d})\log(\frac{\Phi_I}{\psi_{\text{lidar}}})\log(\frac{\Theta_I}{\psi_{\text{lidar}}})\Big) \quad (19)$$

Considering that $\Phi_I$ and $\Theta_I$ are bounded by constants (e.g., $\pi$) and $\psi_{\text{lidar}}$ is a constant parameter related to the LiDAR sensor, the time complexity can be further simplified as $\mathcal{O}(\frac{nR}{d}\log(\frac{D}{d}))$. ∎

### D. Proof of Theorem 3

*Proof.* The occupancy updates involve ray-casting to determine the cells intersected with the rays, followed by updating their occupancy states on the map. We start by considering the worst-case scenario where all points are at the detection range $R$ (as shown in Fig.9(b) in the primary article). In such cases, the number of cells traversed by rays can be approximated as $\mathcal{S}_{\text{rc}} = \mathcal{O}(\frac{nR}{d})$, as derived in (3) and (5). The time complexity of map updates is $\mathcal{O}(1)$ for a grid-based map and $\mathcal{O}(\log(D/d))$ for an octree-based map. Therefore, the time complexity for occupancy updates on a grid-based map and an octree-based map is $\mathcal{O}(\frac{nR}{d})$ and $\mathcal{O}(\frac{nR}{d}\log(\frac{D}{d}))$, respectively. ∎

*E. Proof of Theorem 4*

*Proof.* The query process in D-Map contains a query on the octree and a query on the hashing grid map whose time complexity is $\mathcal{O}(\log(D/d))$ and $\mathcal{O}(1)$, respectively. Therefore, the time complexity of the occupancy state query is easily obtained as $\mathcal{O}(\log(D/d))$. ∎



\end{document}